\documentclass{article}

\usepackage{mystyle}

\usepackage[final,nonatbib]{neurips_2022}

\usepackage[utf8]{inputenc} %
\usepackage[T1]{fontenc}    %
\usepackage{hyperref}       %
\usepackage{url}            %
\usepackage{booktabs}       %
\usepackage{amsfonts}       %
\usepackage{microtype}      %
\usepackage{xcolor}         %

\title{Whitening Convergence Rate of\\Coupling-based Normalizing Flows}

\author{%
	Felix Draxler \\
	Heidelberg University \\
	\texttt{felix.draxler@iwr.uni-heidelberg.de} \\
	\And
	Christoph Schnörr \\
	Heidelberg University \\
	\texttt{schnoerr@math.uni-heidelberg.de} \\
	\And
	Ullrich Köthe \\
	Heidelberg University \\
	\texttt{ullrich.koethe@iwr.uni-heidelberg.de} \\
}

\begin{document}

	\maketitle

	\begin{abstract}
		Coupling-based normalizing flows (e.g.~RealNVP) are a popular family of normalizing flow architectures that work surprisingly well in practice.
		This calls for theoretical understanding. Existing work shows that such flows \textit{weakly} converge to arbitrary data distributions \cite{teshima_coupling-based_2020}.
		However, they make no statement about the stricter convergence criterion used in practice, the maximum likelihood loss.
		For the first time, we make a quantitative statement about this kind of convergence:
		We prove that all coupling-based normalizing flows perform whitening of the data distribution (i.e.~diagonalize the covariance matrix) and derive corresponding convergence bounds that show a linear convergence rate in the depth of the flow.
		Numerical experiments demonstrate the implications of our theory and point at open questions.
	\end{abstract}

	\section{Introduction}
	
	Normalizing flows \cite{kobyzev_normalizing_2021,papamakarios_normalizing_2021} are among the most promising approaches to {\em generative} machine learning and have already demonstrated convincing performance in a wide variety of practical applications, ranging from image analysis \cite{dinh_nice_2015,dinh_density_2017,kingma_glow_2018,mackowiak_generative_2021,ardizzone_training_2020} to astrophysics \cite{ardizzone_analyzing_2018}, mechanical engineering \cite{noevercastelos_model_2022}, causality \cite{muller_learning_2021}, computational biology \cite{noe_boltzmann_2019} and medicine \cite{adler_uncertainty-aware_2019}.
	As the name suggests, normalizing flows represent complex data distributions as %
	bijective transformations (also known as flows or {\em push-forwards}) of standard normal or other well-understood distributions.
	
	In this paper, we focus on a theoretical underpinning of coupling-based normalizing flows, a particularly effective class of normalizing flows in terms of invertible neural networks. %
	All of the above applications are actually implemented using coupling-based normalizing flows.
	Their central building blocks are {\em coupling layers}, which decompose the space into two subspaces called \textit{active} and \textit{passive} subspace (see \cref{sec:fundamentals}). Only the active dimensions are transformed conditioned on the passive dimensions,
	which makes the mapping computationally easy to invert.
	In order to vary the assignment of dimensions to the active and passive subspaces, coupling layers are combined with preceding orthonormal transformation layers into {\em coupling blocks}.
	These blocks are arranged into deep networks such that the orthonormal transformations are sampled uniformly at random from the orthogonal matrices and the coupling layers are trained with the maximum likelihood objective, see \cref{eq:loss}.
	Upon convergence of the training, the sequence of coupling blocks gradually transforms the probability density that generated the given training data, into a standard normal distribution and vice versa.
	
	Since the resulting normalizing flows deviate significantly from \textit{optimal} transport flows \cite{draxler_characterizing_2020} and the bulk of the mathematical literature is focusing on optimal transport, an analysis tailored to coupling architectures is lacking.
	In a landmark paper, \cite{teshima_coupling-based_2020} proved that sufficiently large affine coupling flows %
	weakly converge to arbitrary data densities.
	The notion of weak convergence is critical here, as \textit{it does not imply convergence in maximum likelihood} \cite[Remark 3]{koehler_representational_2021}.
	Maximum likelihood (or, equivalently, the Kullback-Leibler (KL) divergence) is the loss that is actually used in practice. It can be used for gradient descent and it guarantees not only convergence in samples (``$x \sim q(x) \to x \sim p(x)$'') but also in density estimates (``$q(x) \to p(x)$''). It is strong in the sense that the square root of the KL divergence upper bounds (up to a factor 2) the total variation metric, and hence also the Wasserstein metric if the underlying space is bounded \cite{gibbs_choosing_2002}. Moreover, convergence under the KL divergence implies weak convergence which is fundamental for robust statistics \cite{huber_robust_2009}.
	
	We take a first step towards showing that coupling blocks also converge in terms of maximum likelihood.
	To the best of our knowledge, our paper presents for the first time a quantitative convergence analysis of coupling-based normalizing flows based on this strong notion of convergence.

	Specifically, we make the following contributions towards this goal:
	\begin{itemize}
		\item We utilize that the loss of a normalizing flow can be decomposed into two parts (\cref{fig:split-loss-convergence}): The divergence to the nearest Gaussian (\textit{non-Gaussianity}) plus the divergence of that Gaussian to the standard normal (\textit{non-Standardness}).
		\item The contribution of a single coupling layer on the non-Standardness is analyzed in terms of matrix operations (Schur complement and scaling).
		\item Explicit bounds for the non-Standardness after a single coupling block in expectation over all orthonormal transformations are derived.
		\item We use these results to prove that a sequence of coupling blocks whitens the data covariance and to derive linear convergence rates for this process.
	\end{itemize}
	Our results hold for all coupling architectures we are aware of (\cref{app:architectures}), including: NICE \cite{dinh_nice_2015}, RealNVP \cite{dinh_density_2017}, and GLOW \cite{kingma_glow_2018}; Flow++ \cite{ho_flow_2019}; nonlinear-squared flow \cite{ziegler_latent_2019}; linear, quadratic \cite{muller_neural_2019}, cubic \cite{durkan_cubic-spline_2019}, and rational quadratic splines \cite{durkan_neural_2019}; neural autoregressive flows \cite{huang_augmented_2020}, and unconstrained monotonic neural networks \cite{wehenkel_unconstrained_2019}.
	We confirm our theoretical findings experimentally and identify directions for further improvement.
	
	\begin{figure}
		\centering
		\includegraphics[width=\linewidth]{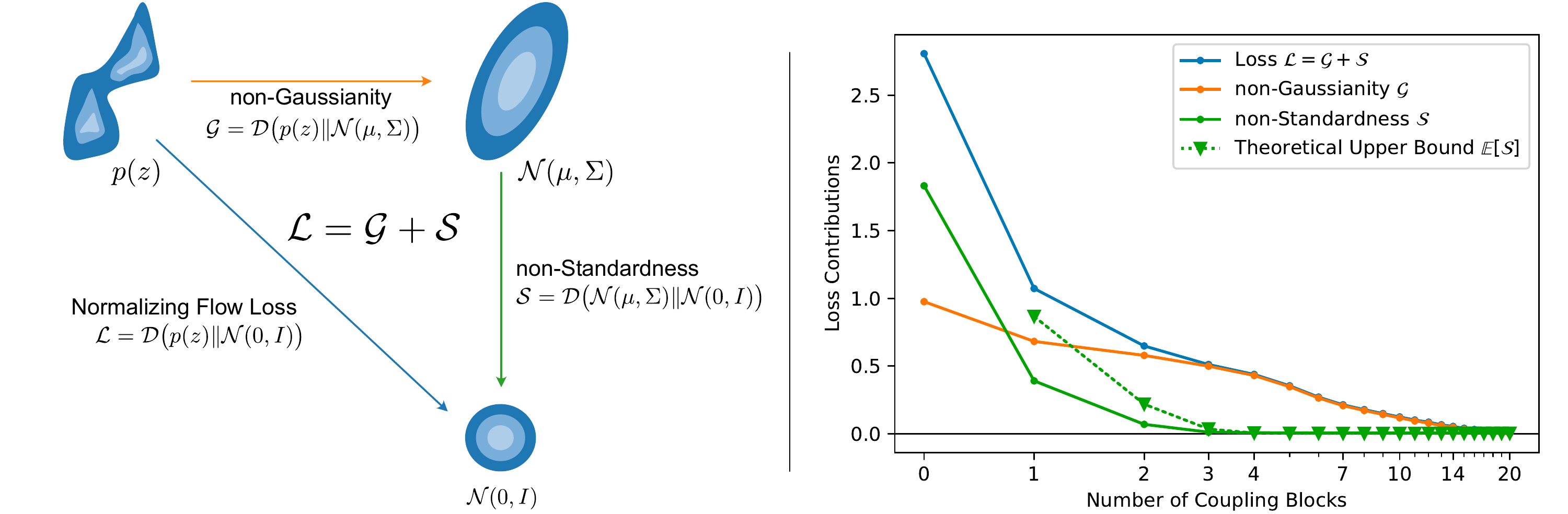}
		\caption{\textit{(Left)} The Maximum Likelihood Loss $\Ll$ {\color{C0}(blue)} can be split into the \textit{non-Gaussianity} $\Gg$ {\color{C1}(orange)} \cite{cardoso_dependence_2003} and the \textit{non-Standardness $\Ss$} {\color{C2}(green)} of the latent code $z = f_\theta(x)$: $\Ll = \Gg + \Ss$ (\cref{prop:pythagorean-identity}). For the latter, we give explicit guarantees as one more coupling block is added in \cref{thm:single-layer-precise,thm:single-layer-guarantee} and show a global convergence rate in \cref{thm:deep-network-guarantee}. \textit{(Right)} Typical fit of EMNIST digits by a standard affine coupling flow for various depths. Our theory (\cref{thm:single-layer-precise}) upper bounds the average $\Ss$ for $L+1$ coupling blocks given a trained model with $L$ coupling blocks {\color{C2}(dotted green)}. We observe that our bound is predictive for how much end-to-end training reduces $\Ss$.}
		\label{fig:split-loss-convergence}
	\end{figure}

	\section{Related work}
	
	Analyzing which distributions coupling-based normalizing flows can approximate is an active area of research. 
	A general statement shows that a coupling-based normalizing flow which can approximate an arbitrary invertible function can learn any probability density \textit{weakly} \cite{teshima_coupling-based_2020}. This applies to affine coupling flows \cite{dinh_nice_2015,dinh_density_2017,kingma_glow_2018}, Flow++ \cite{ho_flow_2019}, neural autoregressive flows \cite{huang_neural_2018}, and SOS polynomial flows \cite{jaini_sum--squares_2019}.
	Affine coupling flows converge to arbitrary densities in Wasserstein distance \cite{koehler_representational_2021}. %
	Both universality results, however, require that the couplings become ill-conditioned (i.e.~the learnt functions become increasingly discontinuous as the error decreases, whereas in practice one observes that functions remain smooth). Also, they consider only a finite subspace of the data space.
	Even more importantly, the convergence criterion employed in their proofs (weak convergence resp.~convergence under Wasserstein metric) is critical: Those criteria do not imply convergence in the loss that is employed in practice \cite[Remark 3]{koehler_representational_2021}, the Kullback-Leibler divergence (equivalent to maximum likelihood). An arbitrarily small distance in any of the above metrics can even result in an infinite KL divergence.
	In contrast to previous work on affine coupling flows, we work directly on the KL divergence. We decompose it in two contributions and show the flow's convergence for one of the parts.
	
	Regarding when ill-conditioned flows need to arise to fit a distribution, \cite{lee_universal_2021} showed that well-conditioned affine couplings can approximate log-concave padded distributions, again in terms of Wasserstein distance.
	Lipschitz flows on the other hand cannot model arbitrary tail behavior, but this can be fixed by adapting the latent distribution \cite{jaini_tails_2020}.
	
	SOS polynomial flows converge in \textit{total variation} to arbitrary probability densities \cite{ishikawa_universal_2022}, which also does not imply convergence in KL divergence; zero-padded affine coupling flows converge weakly \cite{huang_augmented_2020}, and so do Neural ODEs \cite{zhang_approximation_2020,teshima_universal_2020}.
	
	Closely related to our work, 48 linear affine coupling blocks can represent any invertible linear function $Ax + b$ with $\det(A) > 0$ \cite[Theorem 2]{koehler_representational_2021}. This also allows mapping any Gaussian distribution $\Nn(\mean, \Sigma)$ to the standard normal $\Nn(0, I)$. We put this statement into context in terms of the KL divergence: The loss is exactly composed of the divergence to the nearest Gaussian and of that Gaussian to the standard normal. We then make strong statements about the convergence of the latter, concluding that for typical flows a smaller number of layers is required for accurate approximation than predicted by \cite{koehler_representational_2021}.

	\section{Coupling-based normalizing flows}
	\label{sec:fundamentals}
	
	Normalizing flows learn an invertible function $f_\theta(x)$ that maps samples $x$ from some unknown distribution $p(x)$ given by samples to \textit{latent variables} $z = f_\theta(x)$ so that $z$ follow a simple distribution, typically the standard normal. The function $f_\theta$ then yields an estimate $q(x)$ for the true data distribution $p(x)$ via the change of variables formula (e.g.~\cite{dinh_density_2017}):
	\eql{
		q(x)
		= \Nn(f_\theta(x); 0, I) |\det J|,
	}
	where $J = \nabla f_\theta(x)$ is the Jacobian of $f_\theta(x)$. We can train a normalizing flow via the maximum likelihood loss, which is equivalent to minimizing the Kullback-Leibler divergence between the distribution of the latent code $q(z)$, as given by $z = f_\theta(x)$ when $x \sim p(x)$, and the standard normal:
	\eql{
		\label{eq:loss}
		L = \KL{q(z)}{\Nn(0, I)} = \EE_{x \sim p(x)}\left[\tfrac12 \norm{f_\theta(x)}^2 - \log |\det J|\right] + \const.
	}
	The invertible architecture that makes up $f_\theta$ has to (i)~be computationally easy to invert, (ii)~be able to represent complex transformations, and (iii)~have a tractable Jacobian determinant $|\det J|$ \cite{ardizzone_analyzing_2018}.
	Building such an architecture is an active area of research, see e.g.~\cite{kobyzev_normalizing_2021} for a review.
	In this work, we focus on the family of coupling-based normalizing flows, first presented in the form of the NICE architecture \cite{dinh_nice_2015}. It is a deep architecture that consists of several blocks, each containing a rotation, a coupling and an ActNorm layer \cite{kingma_glow_2018}:
	\eql{
		f_\text{block}(x) = (f_\text{act} \circ f_\text{cpl} \circ f_\text{rot})(x).
	}
	The coupling $f_\text{cpl}$ splits an incoming vector $x_0$ in two parts along the coordinate axis: The first part $p_0$, which we call \textit{passive}, is left unchanged. The second part $a_0$, which we call \textit{active}, is modified as a function of the passive dimensions:
	\eql{
		\label{eq:realnvp-layer}
		f_\text{cpl}(x_0) = f_\text{cpl}\begin{pmatrix}
			p_0 \\ a_0
		\end{pmatrix} = \begin{pmatrix}
			p_0 \\ c(a_0; p_0)
		\end{pmatrix} =: \begin{pmatrix}
			p_1 \\ a_1
		\end{pmatrix}.
	}
	Here, the coupling function $c: \RR^{D/2} \times \RR^{D/2} \to \RR^{D/2}$ has to be a function that is easy to invert when $p_0$ is given, i.e.~it is easy to compute $a_0 = c^{-1}(a_1; p_0)$ given $p_0$.
	This makes the coupling easy to invert: Call $x_1 = (p_1; a_1)$ the output of the layer, then $p_0 = p_1$. Use this to invert $a_1 = c(a_0; p_0)$.
	For example, RealNVP \cite{dinh_density_2017} proposes a simple affine transformation for $c$:
	$a_1 = c(a_0; p_0) = a_0 \odot s(p_0) + t(p_0)$ ($\odot$ means element-wise multiplication). $s(p_0) \in \RR^{D/2}_+$ and $ t(p_0) \in \RR^{D/2}$ are computed by a feed-forward neural network.
	The coupling functions $c$ of other architectures our theory applies to are listed in \cref{app:architectures}.
	
	An Activation Normalization (ActNorm) layer \cite{kingma_glow_2018} helps stabilize training and is implemented in practice like in the popular INN framework \texttt{FrEIA} \cite{ardizzone_framework_2018}. It rescales and shifts each dimension:
	\eql{
		f_\text{act}(x) = r \odot x + u,
		\label{eq:act-norm}
	}
	given parameters $r \in \RR^D_+$ and $u \in \RR^D$. We include it as it simplifies our mathematical arguments.
	
	If we were to concatenate several coupling layers, the entire network would never change the passive dimensions apart from the element-wise affine transformation in the ActNorm layer. Here, the rotation layers $f_\text{rot}(x) = Qx$ come into play \cite{kingma_glow_2018}. They multiply an orthogonal matrix $Q$ to the data, changing which subspaces are passive respectively active. This matrix is typically fixed at random at initialization and then left unchanged during training.

	\section{Coupling layers as whitening transformation}
	\label{sec:whitening}
	
	The central mathematical question we answer in this work is the following: How can a deep coupling-based normalizing flow \textit{whiten} the data? As the latent distribution is a standard normal, whitening is a necessary condition for the flow to converge. This is a direct property of the loss:
	\begin{proposition}[Pythagorean Identity, Proof in \cref{app:pythagorean-identity-proof}]
		\label{prop:pythagorean-identity}
		Given data with distribution $p(x)$ with mean $\mean$ and covariance $\Sigma$. Then, the Kullback-Leibler divergence to a standard normal distribution decomposes as follows:
		\eql{
			\KL{p(x)}{\Nn(0, I)}
			= \underbrace{\KL{p(x)}{\Nn(\mean, \Sigma)}}_{\text{non-Gaussianity } \Gg(p)}
			+ \underbrace{\KL{\Nn(\mean, \Sigma)}{\Nn(0, I)}}_{\text{non-Standardness } \Ss(p)}
			\label{eq:gauss-standard-split}
		}
		and the non-Standardness again decomposes:
		\eql{
			\Ss(p)
			= \underbrace{\KL{\Nn(\mean, \Sigma)}{\Nn(\mean, \Diag(\Sigma))}}_{\text{Correlation } \Cc(p)}
			+ \underbrace{\KL{\Nn(\mean, \Diag(\Sigma))}{\Nn(0, I)}}_{\text{Diagonal non-Standardness}}.
			\label{eq:diagonal-standard-split}
		}
	\end{proposition}
	This splits the transport from the data distribution to the latent standard normal into three parts: (i) From the data to the nearest Gaussian distribution $\Nn(\mean, \Sigma)$, measured by $\Gg$. (ii) From that nearest Gaussian to the corresponding uncorrelated Gaussian $\Nn(\mean, \Diag(\Sigma))$, measured by $\Cc$. (iii) From the uncorrelated Gaussian to standard normal.
	
	We do not make explicit use of the fact that the \textit{non-Standardness} can again be decomposed, but we show it nevertheless to relate our result to the literature:
	The Pythagorean identity $\KL{p(x)}{\Nn(\mean, \Diag(\Sigma))} = \Gg(p) + \Cc(p)$ has been shown before by \cite[Section 2.3]{cardoso_dependence_2003}. Both their and our result are specific applications of the general \cite[Theorem 3.8]{amari_methods_2007} from information geometry. Our proof is given in \cref{app:pythagorean-identity-proof}.
	
	\cref{prop:pythagorean-identity} is visualized in \cref{fig:split-loss-convergence}. In an experiment, we fit a set of Glow \cite{kingma_glow_2018} coupling flows of increasing depths to the EMNIST digit dataset \cite{cohen_emnist_2017} using maximum likelihood loss and measure the capability of each flow in decreasing $\Gg$ and $\Ss$ (Details in \cref{app:exp-split-loss-convergence}). The form of the non-Standardness $\Ss$ is given by the well-known KL divergence between the involved normal distributions, see \cref{eq:kl-gaussians} in \cref{app:pythagorean-identity-proof}. It is invariant under rotations $Q$ and only depends on the first two moments $\mean, \Sigma$:
	\eql{
		\label{eq:non-standardness}
		\Ss(\mean, \Sigma) := \Ss(p) = \frac12 (\norm{\mean}^2 + \tr \Sigma - D - \log \det \Sigma) ) = \Ss(Q \mean, Q \Sigma Q^\transy).
	}
	The non-Standardness $\Ss$ will be our measure on how far the covariance and mean have approached the standard normal in the latent space. We give explicit loss guarantees for $\Ss$ for a single coupling block in \cref{thm:single-layer-precise,thm:single-layer-guarantee} and imply a linear convergence rate for a deep network in \cref{thm:deep-network-guarantee}.
	
	Deep Normalizing Flows are typically trained end-to-end, i.e.~the entire stack of blocks is trained jointly. In this work, our ansatz is to consider the effect of a single coupling block on the non-Standardness $\Ss$. Then, we combine the effect of many isolated blocks, disregarding potential further improvements to $\Ss$ due to joint, cooperative learning of all blocks.
	This simplifies the theoretical analysis of the network, but it is not a restriction on the model: Any function that is achieved in block-wise training could also be the solution of end-to-end training.
	
	We aim to strongly reduce $\Ss$ while leaving room for a complementary theory explaining how non-Gaussianity $\Gg$ is reduced in practice. %
	Note that affine-linear functions $Ax + b$ can never change $\Gg$, because they jointly transform the distribution $p(x)$ at hand and correspondingly the closest Gaussian to it (see \cref{lem:linear-constant-non-gaussianity} in \cref{app:single-layer-whitening-proof}).
	Thus, if we restrict our coupling layers to be affine-linear functions, we are able to reduce $\Ss$ without increasing $\Gg$ in turn.
	This motivates considering affine-linear couplings of the following form, spelled out together with ActNorm as given by \cref{eq:act-norm}. \textbf{The results in this work apply to all coupling architectures}, as they all can represent this coupling, see \cref{app:architectures}.
	\eql{
		\begin{pmatrix}
			p_1 \\
			a_1
		\end{pmatrix} = (f_\text{act} \circ f_\text{cpl})(Q x)
		= r \odot \begin{pmatrix}
			I & 0 \\
			T & I
		\end{pmatrix} \begin{pmatrix}
			p_0 \\
			a_0
		\end{pmatrix} + u.
		\label{eq:affine-linear-coupling}
	}
	For future work considering $\Gg$, we propose to lift the restriction to affine-linear layers while making sure that $\Ss$ behaves as described in what follows. As the convergence of $\Gg$ however will strongly depend on the coupling architecture and data $p(x)$ at hand, this is beyond the scope of this work.
	
	Our first result shows which mean $\mean_1$ and covariance $\Sigma_1$ a single affine-linear coupling as in \cref{eq:affine-linear-coupling} yields to minimize $\Ss(\mean_1, \Sigma_1)$ given data with mean $\mean$ and covariance $\Sigma$, rotated by $Q$:
	\begin{proposition}[Proof in \cref{app:single-layer-whitening-proof}]
		\label{prop:single-layer-whitening}
		Given $D$-dimensional data with mean $\mean$ and covariance $\Sigma$ and a rotation matrix $Q$. Split the covariance of the rotated data into four blocks, corresponding to the passive and active dimensions of the coupling layer:
		\eql{
			\label{eq:covariance-int}
			Q \Sigma Q^\transy = \Sigma_0 = \begin{pmatrix}
				\Sigma_{0, pp} & \Sigma_{0, pa} \\
				\Sigma_{0, ap} & \Sigma_{0, aa}
			\end{pmatrix}
		}
		Then, the moments $\mean_1, \Sigma_1$ that can be reached by a coupling as in \cref{eq:affine-linear-coupling} are:
		\eql{
			\label{eq:covariance-out}
			\mean_1 = 0, \qquad \Sigma_1 = \begin{pmatrix}
				M(\Sigma_{0, pp}) & 0 \\
				0 & M(\Sigma_{0, aa} - \Sigma_{0, ap} \Sigma_{0, pp}^{-1} \Sigma_{0, pa})
			\end{pmatrix}.
		}
		This minimizes $\Ss$ as given in \cref{eq:non-standardness}, and $\Gg$ does not increase.
	\end{proposition}
	The function $M$ takes a matrix $A$ and rescales the diagonal to $1$ as follows. It is a well-known operation in numerics called Diagonal scaling or Jacobi preconditioning so that $M(A)_{ii} = 1$:
	\eql{
		\label{eq:diagonal-preconditioning}
		M(A)_{ij} = \sqrt{A_{ii}A_{jj}}^{-1} A_{ij} = (\Diag(A)^{-1/2} A \Diag(A)^{-1/2})_{ij}.
	}
	\cref{prop:single-layer-whitening} shows how the covariance can be brought closer to the identity.
	
	The new covariance has passive and active dimensions uncorrelated. In the active subspace, the covariance is the Schur complement $\Sigma_{0, aa} - \Sigma_{0, ap} \Sigma_{0, pp}^{-1} \Sigma_{0, pa}$. This coincides with the covariance of the Gaussian $\Nn(0, \Sigma)$ as it is conditioned on any passive value $p$. Afterwards, the diagonal is rescaled to one, matching the standard deviations of all dimensions with the desired latent code. %
	The proof is based on a more general result how a single layer maximally reduces the Maximum Likelihood Loss for arbitrary data \cite{draxler_characterizing_2020}, which we apply to the non-Standardness $\Ss$ (see \cref{app:single-layer-whitening-proof}).
	
	\cref{fig:emnist-single-layer-covariance} shows an experiment in which a single affine-linear layer was trained to bring the covariance of EMNIST digits \cite{cohen_emnist_2017} as close to $\one$ as possible (Details in \cref{app:exp-covariance}). The experimental result coincides with the prediction by \cref{prop:single-layer-whitening}. Due to the finite batch-size, a small difference between theory and experiment remains.%
	
	\begin{figure}
		\centering
		\includegraphics[width=\linewidth,trim=0 .15cm 0 0,clip]{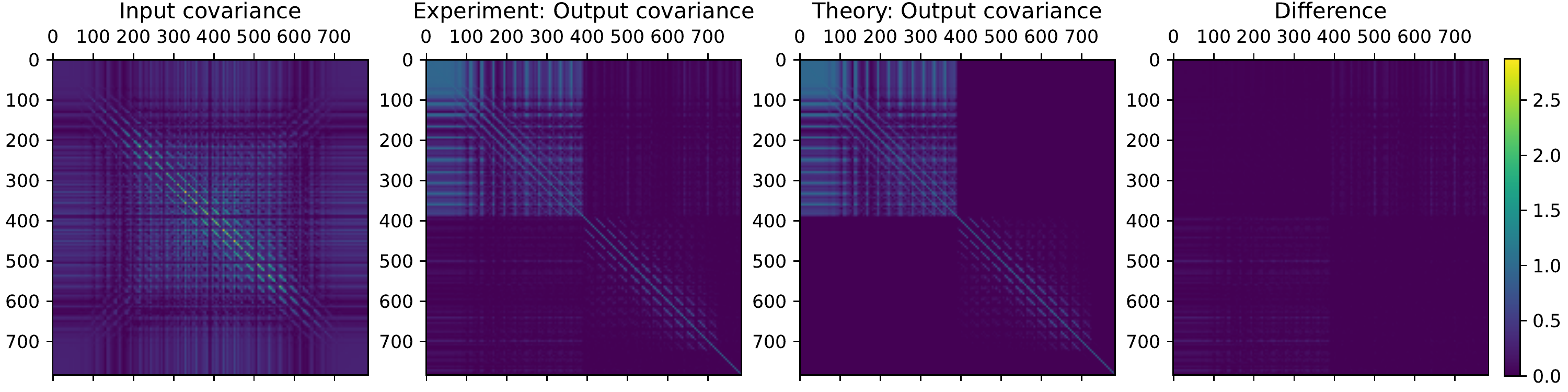}
		\caption{\textbf{How a single coupling layer can whiten the covariance} at the example of the EMNIST digits covariance matrix \textit{(first panel)}. The covariance after a single layer trained experimentally to minimize non-Standardness $\Ss(m_1, \Sigma_1)$ \textit{(second panel)}, which matches closely the prediction of \cref{prop:single-layer-whitening} \textit{(third panel)}. The difference between theory and experiment vanishes \textit{(last panel)}.}
		\label{fig:emnist-single-layer-covariance}
	\end{figure}

	\section{Explicit convergence rate}
	
	In \cref{sec:whitening}, we showed how a single coupling layer acts on the first two moments of a given data distribution to whiten it. We now explicitly demonstrate how much progress this means in terms of the non-Standardness $\Ss(\mean_1, \Sigma_1)$ averaged over rotations $Q$ (\cref{thm:single-layer-precise,thm:single-layer-guarantee}) and show the consequences for multiple blocks (\cref{thm:deep-network-guarantee}).

	\subsection{Single coupling block guarantees}
	\label{sec:single-block-guarantees}
	
	\cref{prop:single-layer-whitening} allows the computation of the minimum non-Standardness after a single coupling block given its rotation $Q$, by evaluating $\Ss(\mean_1, \Sigma_1)$. In fact, if we were to choose $Q$ such that the data is rotated so that principal components lie on the axes (i.e.~obtain $Q$ using PCA), a single coupling block suffices to reduce the covariance to the identity: $\Sigma_0 = Q \Sigma Q^\transy$ would be a diagonal matrix and $\Sigma_1 = I$. This is not the case in practice, where this optimal orientation has zero probability: $Q$ is chosen uniformly at random before training from all orthogonal matrices.
	One could argue that one should whiten the data before passing it to the flow, reducing $\Ss$ to zero from the start. However, any change in the architecture could possibly alter the performance of the network with regard to reducing the non-Gaussianity $\Gg$. Also, our work shows that coupling-based normalizing flows are already well-equipped to bring the non-Standardness to zero without such modifications.
	To properly describe the achievable non-Standardness $\Ss$, \textbf{we formulate all guarantees as expectations over the rotation} $Q$, corresponding to the loss averaged over training runs.
	
	We make two mild assumptions on our data that are part of usual data-preprocessing, when the mean is subtracted from the data and all data points are divided by the scalar $\sqrt{\tr\Sigma/D}$ (not to be confused with diagonal preconditioning, which acts dimension-wise).
	\begin{assumption}
		\label{as:centered}
		The data $p(x)$ is centered: $\EE_{x \sim p(x)}[x] = 0$.
	\end{assumption}
	\begin{assumption}
		\label{as:normalized-covariance}
		The covariance is normalized: $\tr\Sigma = D$.
	\end{assumption}
	The assumptions simplify the non-Standardness in \cref{eq:non-standardness}, which now only depends on the determinant of $\Sigma$:%
	\eql{
		\label{eq:loss-before}
		\Ss(\Sigma) = -\tfrac12 \log\det\Sigma = -\tfrac12 \log\det\Sigma_0 = \Ss(\Sigma_0)
	}
	for arbitrary rotation $Q$. We aim to compute the average non-Standardness after a single block $\EE_{Q \in p(Q)}[\Ss(\Sigma_1(Q))]$. For any $Q$, $\Ss(\Sigma_1)$ is again given by the determinant of the covariance $\Sigma_1(Q)$ as \cref{as:centered,as:normalized-covariance} remain fulfilled: By \cref{prop:single-layer-whitening} $\mean_1 = 0$ and the diagonal preconditioning $M$ ensures that the trace of $\Sigma_1$ is $D$.
	We write $\det(\Sigma_1)$ via $M_a$ and $M_p$, the diagonal matrices that make up the diagonal preconditioning in \cref{eq:diagonal-preconditioning}, and use the Schur determinantal formula for the determinant of block matrices: $\det(\Sigma_{0,pp}) \det(\Sigma_{0,aa} - \Sigma_{0,ap}\Sigma_{0,pp}^{-1}\Sigma_{0,pa}) = \det(\Sigma_0) = \det(\Sigma)$ 
	\cite{horn_matrix_2012}.
	We thus get $\det(\Sigma_1) = \det(M_p\Sigma_{0,pp}M_p) \det(M_a(\Sigma_{0,aa} - \Sigma_{0,ap}\Sigma_{0,pp}^{-1}\Sigma_{0,pa})M_a) = \det(M_p^2) \det(M_a^2) \det(\Sigma)$. 
	Inserting this into \cref{eq:loss-before}, we find:
	\eql{
		\label{eq:non-standardness-after}
		\Ss(\Sigma_1) 
		= -\tfrac12 ( \log\det\Sigma + \log \det M_p^2 + \log \det M_a^2 ) \leq \Ss(\Sigma_0) = \Ss(\Sigma).
	}
	The inequality $\Ss(\Sigma_1) \leq \Ss(\Sigma_0)$ holds because $\Sigma_1 = \Sigma_0$ is an admissible solution of the coupling layer optimization, but $\Sigma_1$ as given by \cref{prop:single-layer-whitening} is a minimizer of $\Ss(\Sigma_1)$.
	
	We average this quantity over training runs, i.e.~over rotations $Q$:
	\eql{
		\label{eq:expected-non-standardness-after}
		\EE_{Q \sim p(Q)}[\Ss(\Sigma_1)] = -\tfrac12 \big( \log\det \Sigma + \EE_{Q \sim p(Q)}[\log \det M_p^2] + \EE_{Q \sim p(Q)}[\log \det M_a^2] \big).
	}
	The main difficulty lies in the computation of $\EE_{Q \sim p(Q)}[\log \det M_a^2]$. Here, we contribute the two strong statements \cref{thm:single-layer-precise,thm:single-layer-guarantee} below.

	\subsubsection{Precise guarantee}
	
	The first result relies on projected orbital measures as developed by \cite{olshanski_projections_2013}. This theory describes the eigenvalues of submatrices of matrices in a random basis. We require such a result for integrating over $p(Q)$ in $\EE_{Q \sim p(Q)}[\log \det M_a^2]$. In contrast to typical choices of $p(Q)$, the theory to this date only covers data rotated by unitary matrices.\footnote{The only result known to us would yield predictions for $D=2$ \cite{faraut_rayleigh_2015}, whereas we are interested in large $D$.} To comply with \cite{olshanski_projections_2013}, we make two more assumptions:
	\begin{assumption}
		\label{as:unitary-rotation}
		The distribution of rotations is the Haar measure over \emph{unitary} matrices $U(D)$.
	\end{assumption}
	\begin{assumption}
		\label{as:non-degenerate-eigenvalues}
		The eigenvalues of the covariance matrix $\Sigma$ are distinct: $\lambda_i \neq \lambda_j$ for $i \neq j$.
	\end{assumption}
	One could think that the step from orthogonal to unitary rotations takes us far away from the scenario we want to consider. We will later observe empirically that the difference between averaging over unitary and orthogonal matrices is negligible. Technically, the covariance matrix remains positive definite, so the non-Standardness $\Ss$ is always real (see \cref{app:imaginary-part}). We will write $\EE_{Q \sim U(D)}[\,\cdot\,]$ to denote expectations over unitary matrices.
	
	\cref{as:non-degenerate-eigenvalues} is typically satisfied when working with real data that are in `general position'. %
	We are now ready to \textbf{compute the average training performance} of a single coupling block:
	\begin{theorem}[Proof in \cref{app:single-layer-precise-proof}]
		\label{thm:single-layer-precise}
		Given $D$-dimensional data with covariance $\Sigma$ with eigenvalues $\lambda_1, \dots \lambda_D$. Assume that \cref{as:centered,as:normalized-covariance,as:unitary-rotation,as:non-degenerate-eigenvalues} hold.
		Then, after a single coupling block, the expected non-Standardness is bounded from above:
		\eql{
			\!\!\!\EE_{Q \in U(D)}[\Ss(\Sigma_1(Q))] 
			\!<\! \Ss(\Sigma) \!+\! \tfrac{D}2 \log\!\bigg(\!\!(-1)^{\tfrac{D}{2}+1} \sum_{i=1}^D \lambda_i^{1-\tfrac{D}{2}} \log(\lambda_i) R(\lambda_i^{-1}; \lambda_{\neq i}^{-1}) e_{\tfrac{D}{2}-1}(\lambda_{\neq i}^{-1}) \!\!\bigg).
			\label{eq:single-layer-precise}
		}
		Here, $\lambda_{\neq i} := \{ \lambda_1, \dots, \lambda_{i-1}, \lambda_{i+1}, \dots, \lambda_D \}$ and $R, e_{K}$ are given by:
		\eql{
			\label{eq:inverse-volume-elementary-symmetric}
			R(a; \{b_i\}_{i=1}^N) = \prod_{i=1}^N \frac1{a - b_i}
			\quad\text{and}\quad
			e_K(\{b_i\}_{i=1}^N) = \sum_{0 < i_1 < \dots < i_K \leq N} b_{i_1} \cdots b_{i_K}.
		}
	\end{theorem}
	Inequality \eqref{eq:single-layer-precise} sharply bounds the expected non-Standardness that can be achieved by a single block. \changed{The only approximation made is an inequality which comes close to equality as the dimension $D$ increases due to the concentration of the corresponding probability distribution.}
	
	\cref{fig:single-layer-experiment} shows an experiment confirming \cref{thm:single-layer-precise} (Details in \cref{app:exp-single-layer}). We start with covariance matrices using parametrized eigenvalue spectra. On each, we first apply a single coupling block with random Q and train the coupling that maximally reduces $\Ss$ (\cref{prop:single-layer-whitening}). Then we iteratively append 32 additional blocks in the same manner, building a flow of that depth.
	We average the resulting empirical ratio $\Ss(\Sigma_1) / \Ss(\Sigma)$ over several \textit{orthogonal} orientations $Q$ of the rotation layer for each input covariance matrix. Then, we compare this to (i)~experimentally averaging over \textit{unitary} rotations and (ii)~to the prediction by \cref{thm:single-layer-precise} and confirm that it is a valid and close upper bound.
	Details for replication and more examples can be found in \cref{app:exp-single-layer}.
	
	The proof explicitly integrates $\EE[M_a^2]$ using \cite{olshanski_projections_2013} (see \cref{app:single-layer-precise-proof}).
	Numerically evaluating \cref{eq:single-layer-precise} can be hard even for small $D$ as the summands scale as $\Oo(\exp(D))$, but the overall sum scales as $\Oo(D)$. High values cancel due to $R$ alternating in sign, and one requires arbitrary-precision floating point software to evaluate \cref{eq:single-layer-precise}.
	
	\begin{figure}
		\centering
		\includegraphics[width=\linewidth,trim=0 .3cm 0 .25cm,clip]{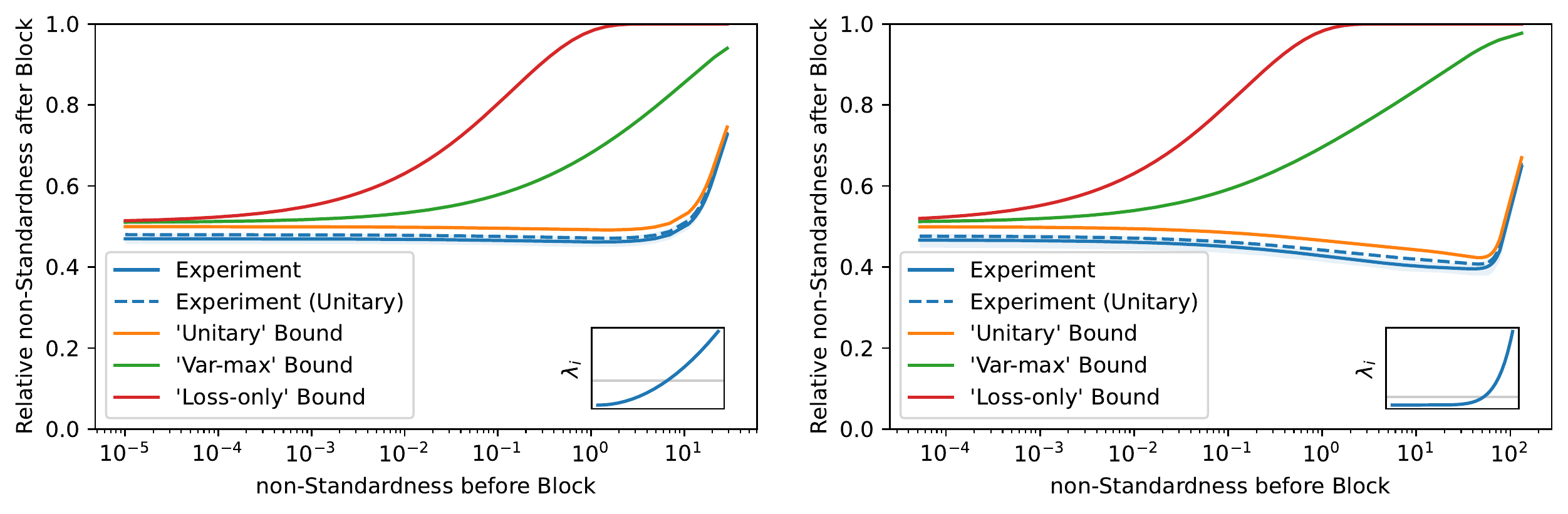}
		\caption{\textbf{Comparison between predicted non-Standardness and experiment} for 48-dimensional parametrized eigenvalue spectra \textit{(insets)}, varied over a parameter which controls the spread of the spectrum and thus changes $\Ss$.
			The experimental average over \textit{orthogonal} rotations matrices {\color{C0}(blue, shaded by Interquartile Range IQR)} is closely matched by the experimental average over \textit{unitary} matrices {\color{C0} (dotted blue)}. The prediction by \cref{thm:single-layer-precise} is a close upper bound that closely matches the experimental behavior {\color{C1} (orange)}.
			The predictions by \cref{thm:single-layer-guarantee} are less precise, but converge to the same value as the precise bound for covariances close to the identitiy: `$\Var$-$\max$' is \cref{eq:single-layer-guarantee-var-max} {\color{C2} (green)} and `Loss-only' is \cref{eq:single-layer-guarantee-loss} {\color{C3} (red)}. More details and examples in \cref{app:exp-single-layer}.}
		\label{fig:single-layer-experiment}
	\end{figure}

	\subsubsection{Interpretable guarantee}
	
	The guarantee in \cref{thm:single-layer-precise} yields useful predictions, but it does not lend itself to further analysis: How does the bound behave over several coupling blocks? What is the behavior for varying dimension $D$?
	Also, \cref{as:unitary-rotation} restricts formal reasoning as we are interested in averaging over orthogonal and not unitary rotations.
	Our second single-block guarantee depends only on simple metrics of the covariance. Moreover, we drop \cref{as:unitary-rotation,as:non-degenerate-eigenvalues}, averaging over orthogonal, not unitary, $Q$:
	\begin{theorem}[Proof in \cref{app:single-layer-guarantee-proof}]
		\label{thm:single-layer-guarantee}
		Given $D$-dimensional data fulfilling \cref{as:centered,as:normalized-covariance} with covariance $\Sigma \neq I$ with eigenvalues $\lambda_1, \dots \lambda_D$.
		Then, after a single coupling block, the expected loss can be bounded from above:
		\eqal{
			\EE_{Q \in O(D)}[\Ss(\Sigma_1(Q))] &
			\leq \Ss(\Sigma) + \frac{D}{4} \log\left(1 - \frac{D^2}{2(D-1)(D+2)} \frac{\Var[\lambda]}{\lambda_{\max}} \right) \label{eq:single-layer-guarantee-var-max} \\&
			\leq \Ss(\Sigma) + \frac{D}{4} \log\left(1 - \frac{D^2}{(D-1)(D+2)} \frac{1 - \sqrt{1 - g^D}}{1 + \sqrt{1 - g^D}} (1 - g) \right) \label{eq:single-layer-guarantee-loss} 
			< \Ss(\Sigma).
		}
		Here, $g$ is the geometric mean of the eigenvalues: $g = \prod_{i=1}^D\lambda_i^{1/D} = \exp(-2\Ss(\Sigma)/D) < 1$ which is a bijection of $\Ss(\Sigma)$.
	\end{theorem}
	These two new bounds on the average achievable non-Standardness $\Ss$ after a single block are also depicted in \cref{fig:single-layer-experiment}. They make useful predictions, but are less precise than \cref{thm:single-layer-precise}. The second bound will be especially useful in what follows because it only depends on the non-Standardness before the block $\Ss(\Sigma)$.
	
	The full proof is given in \cref{app:single-layer-guarantee-proof}. It relies on the integration of monomials of entries of random orthogonal matrices as described by \cite{gorin_integrals_2002} and the arithmetic mean-geometric mean inequality by \cite{cartwright_refinement_1978}.
	
	The first bound suggests an important property of the non-Standardness convergence of a coupling-based normalizing flow in terms of dimension: The performance only marginally depends on the dimension. To see this, divide \cref{eq:single-layer-guarantee-var-max} by $D$ to obtain a statement about the non-Standardness per dimension $\Ss/D$. Then take several data sets with different dimension but same spectrum characteristics (i.e.~same geometric mean, variance and maximum of covariance eigenvalues). The guarantee is then approximately constant in $D$ (it varies slightly with $D^2 / (D^2+D-2)$, which is always close to $1$).

	\subsection{Deep network guarantee}
	\label{sec:deep-network-guarantee}
	
	The previous \cref{sec:single-block-guarantees} was concerned with determining how much a \emph{single} coupling block can typically contribute towards reducing the $\Ss$ to zero.
	Now, we extend this result to compute the expected non-Standardness after a \textit{deep} coupling-based normalizing flow as an explicit function of the number of blocks. We again treat the rotation layer of each block as a random variable, as it is randomly determined before training.
	
	We find that the \textbf{convergence rate} of the covariance to the identity is (at least) \textbf{linear}:
	\begin{theorem}[Proof in \cref{app:deep-network-guarantee-proof}]
		\label{thm:deep-network-guarantee}
		Given $D$-dimensional data fulfilling \cref{as:centered,as:normalized-covariance} with covariance $\Sigma$.
		Then, after $L$ coupling blocks, the expected loss is smaller than:
		\eql{
			\EE_{Q_1, \dots, Q_L \in O(D)}[\Ss(\Sigma_L)] \leq \gamma\left(\Ss(\Sigma)\right)^L \Ss(\Sigma),
		}
		where the convergence rate depends on the non-Standardness before training:
		\eql{
			\label{eq:convergence-rate}
			\gamma(\Ss) = 1 + \tfrac{1}{4 \Ss/D} \log\left(1 - \frac{D^2}{(D-1)(D+2)} \frac{1 - \sqrt{1 - g(\Ss)^D}}{1 + \sqrt{1 - g(\Ss)^D}} \left(1 - g(\Ss) \right) \right) < 1.
		}
	\end{theorem}
	The non-Standardness decreases at least exponentially fast in the number of blocks. The convergence rate that holds for a deep network is computed using the non-Standardness of the input data $\Ss(\Sigma)$. This rate comes from \cref{eq:single-layer-guarantee-loss}. The proof uses that $\gamma(\Ss)$ improves from block to block as $\Ss$ decreases (see \cref{app:deep-network-guarantee-proof}). Again, $g(\Ss) = \exp(-2 \Ss/D) < 1$ is the geometric mean of eigenvalues of $\Sigma$, which increases from block to block.
	
	\begin{figure}
	\centering
	\includegraphics[width=\linewidth,trim=0 .25cm 0 .25cm,clip]{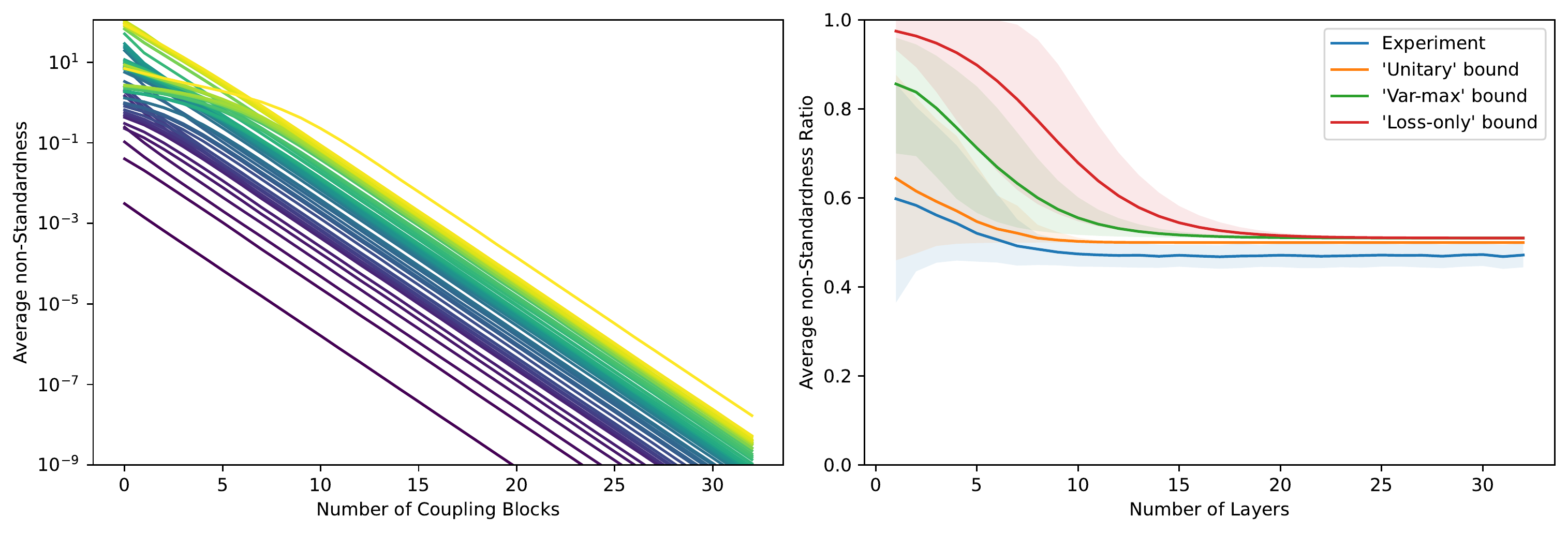}
	\caption{\textbf{Deep network convergence of covariance on toy dataset.} \textit{(Left)} Each line shows the experimental convergence of $\Ss$ via the repeated application of \cref{prop:single-layer-whitening}, averaged over 32 runs with different rotations $Q$. \textit{(Right)} The empirical convergence rate {\color{C0} (blue)}, i.e.~the ratio of $\Ss$ before and after a block, is correctly bounded from above by our predictions in \cref{thm:single-layer-precise} {\color{C1} (orange)}, and the bounds in \cref{thm:single-layer-guarantee}: \cref{eq:single-layer-guarantee-var-max} {\color{C2} (green)} and \cref{eq:single-layer-guarantee-loss} {\color{C3} (red)}. The solid lines show the ratio (bounds) averaged over the toy dataset and rotations, the shade is the IQR. The experiment suggests that a convergence rate like \cref{thm:deep-network-guarantee} can also be derived for the remaining bounds.}
	\label{fig:multi-layer}
\end{figure}

	\cref{fig:multi-layer} shows the convergence of the non-Standardness to zero in an experiment. We build a toy dataset of various covariances where we aim to capture a plethora of possible cases (see \cref{app:exp-multi-layer}). We apply a single coupling block with random $Q$ and the coupling that maximally reduces $\Ss$ via \cref{prop:single-layer-whitening}. We iteratively add such blocks 32 times, building a flow of that depth. The resulting convergence of $\Ss$ as a function of depth is averaged over 32 runs with different rotations.
	The measured curve confirms \cref{thm:deep-network-guarantee}. We find that the rate $\gamma$ in \cref{eq:convergence-rate} is correct, but several experiments show even faster convergence in practice.
	Indeed, the experiments suggest that dividing all upper bounds for $\EE[\Ss(\Sigma_1)]$ in \cref{thm:single-layer-precise,thm:single-layer-guarantee} by $\Ss(\Sigma)$ also bounds the non-Standardness ratio for subsequent blocks. Formally, we conjecture that $\EE[\Ss(\Sigma_L)] / \Ss(\Sigma) \leq (B / \Ss(\Sigma))^L$ where $B$ is the rhs.~of \cref{eq:single-layer-precise,eq:single-layer-guarantee-var-max} (\cref{thm:deep-network-guarantee} shows exactly this for \cref{eq:single-layer-guarantee-loss}). We leave a proof or falsification of this conjecture open to future work.
	
	The experiment also suggests that all bounds agree after a few blocks, leaving a small gap to the experiment. We can explicitly compute this limit value of $\gamma(\Ss)$ by taking $\Ss \to 0$:
	\eql{
		\gamma(\Ss) \xrightarrow{\Ss \to 0} \tfrac{D (D + 2) - 4}{2(D-1)(D+2)} \in \left[1/2, 5/9 \right].
	}
	The two experimental observations together with this limit value suggest the heuristic that \textbf{a single additional coupling block typically reduces the non-Standardness $\Ss$ by a factor of approximately 50\%} if previous blocks are left unchanged, and possibly faster if cooperations between blocks are taken into account.

	\section{Conclusion}
	\label{sec:conclusion}
	
	To the best of our knowledge, this is the first work on coupling-based normalizing flows that provides a quantitative convergence analysis in terms of the KL divergence.
	Specifically, a minimal convergence rate is established at which flows whiten the covariance of the input data under this strong measure of discrepancy of probability distributions. Splitting the loss into the non-Gaussianity $\Gg$ and the non-Standardness $\Ss$, we show that this whitening is a necessary condition for the flow to converge and give explicit guarantees.
	Our derivations suggest the rule of thumb that $\Ss$ can typically be reduced by about 50\% per coupling block.
	
	Our central idea was to separate out the contribution a single isolated block can make to reduce the loss, arguing that end-to-end training can only outperform the concatenation of isolated blocks.
	
	Having separated the tasks a normalizing flow has to solve, and having explained how the non-Standardness $\Ss$ can be reduced to zero, we hope that explaining also the entire convergence of $\Ll = \Gg + \Ss$ with respect to the KL divergence is within reach. In particular, our theory did not yet explore how the non-linear part of each coupling block reduces the non-Gaussianity $\Gg$.

	\begin{ack}
		\affExcellence{work}
		It is also supported by the Vector Stiftung in the project TRINN (P2019-0092).
		
		We thank our colleagues (in alphabetical order) Marcel Meyer, Jens Müller, Robert Schmier and Peter Sorrenson for their help and fruitful discussions.
	\end{ack}

	\bibliographystyle{unsrt}
	\bibliography{bibliography}

	\section*{Checklist}
	
	\begin{enumerate}

		\item For all authors...
		\begin{enumerate}
			\item Do the main claims made in the abstract and introduction accurately reflect the paper's contributions and scope?
			\answerYes{}
			\item Did you describe the limitations of your work?
			\answerYes{see \cref{sec:conclusion}.}
			\item Did you discuss any potential negative societal impacts of your work?
			\answerYes{Generative modeling, which this paper aims to improve, can be used in harmful ways to generate Deepfakes for disinformation.}
			\item Have you read the ethics review guidelines and ensured that your paper conforms to them?
			\answerYes{}
		\end{enumerate}

		\item If you are including theoretical results...
		\begin{enumerate}
			\item Did you state the full set of assumptions of all theoretical results?
			\answerYes{}
			\item Did you include complete proofs of all theoretical results?
			\answerYes{Full proofs are in \cref{app:proofs}.}
		\end{enumerate}

		\item If you ran experiments...
		\begin{enumerate}
			\item Did you include the code, data, and instructions needed to reproduce the main experimental results (either in the supplemental material or as a URL)?
			\answerYes{See \cref{app:experiment-details}.}
			\item Did you specify all the training details (e.g., data splits, hyperparameters, how they were chosen)?
			\answerYes{See \cref{app:experiment-details}.}
			\item Did you report error bars (e.g., with respect to the random seed after running experiments multiple times)?
			\answerYes{}
			\item Did you include the total amount of compute and the type of resources used (e.g., type of GPUs, internal cluster, or cloud provider)?
			\answerYes{See \cref{app:experiment-details}.}
		\end{enumerate}

		\item If you are using existing assets (e.g., code, data, models) or curating/releasing new assets...
		\begin{enumerate}
			\item If your work uses existing assets, did you cite the creators?
			\answerYes{}
			\item Did you mention the license of the assets?
			\answerNA{}
			\item Did you include any new assets either in the supplemental material or as a URL?
			\answerNA{}
			\item Did you discuss whether and how consent was obtained from people whose data you're using/curating?
			\answerNA{}
			\item Did you discuss whether the data you are using/curating contains personally identifiable information or offensive content?
			\answerNA{}
		\end{enumerate}

		\item If you used crowdsourcing or conducted research with human subjects...
		\begin{enumerate}
			\item Did you include the full text of instructions given to participants and screenshots, if applicable?
			\answerNA{}
			\item Did you describe any potential participant risks, with links to Institutional Review Board (IRB) approvals, if applicable?
			\answerNA{}
			\item Did you include the estimated hourly wage paid to participants and the total amount spent on participant compensation?
			\answerNA{}
		\end{enumerate}

	\end{enumerate}

	\clearpage
	\appendix
	
	{\Large \bf \thetitle: Appendix}
	
	\section{Details on Experiments}
	\label{app:experiment-details}
	
	All experiments were carried out on a single AMD Ryzen 7 3700X 8-Core Processor together with a NVIDIA GeForce RTX 2080. At \url{https://github.com/VLL-HD/Coupling-Flow-Bound} we have made available the code for all experiments.
	
	\subsection{Deep network on EMNIST}
	\label{app:exp-split-loss-convergence}
	
	In this experiment, we estimate the capability of affine normalizing flows in reducing the non-Standardness $\Ss$ (see \cref{eq:non-standardness}) as a function of the number of layers.
	We compare this to the theoretic bound in \cref{thm:single-layer-precise}.
	
	To this end, we train affine normalizing flows on EMNIST digits \cite{cohen_emnist_2017}. We leverage a 20-block Glow architecture as described in \cref{sec:fundamentals}. To measure the effect of depth $L = 1, \dots, 20$ of the flow on $\Ss$, we truncate the architecture to $L$ layers.
	
	The architecture is built as follows: We start by down-sampling the input image from gray scale $1 \times 28 \times 28$ to $4 \times 14 \times 14$: Each group of four neighboring pixels is reordered into one pixel with four times the channels in a checkerboard-like pattern.
	Then, eight convolutional coupling blocks with 16 hidden channels are applied.
	They are followed by another down-sampling to $16 \times 7 \times 7$ and eight convolutional coupling blocks with 32 hidden channels.
	After flattening the input, four fully-connected affine coupling blocks are added with 392 hidden dimensions.
	
	When truncating this architecture, we remove blocks \textit{from the left}. For example, when one block is present ($L=1$), only the last coupling block with the fully connected subnetwork remains. This makes the theory in this paper applicable, as \cref{prop:single-layer-whitening} assumes that the neural networks $s$ and $t$ are fully connected (otherwise, the whitening operation cannot always be represented).
	
	We train each depth from scratch for 300 epochs using Adam with a learning rate of $3 \cdot 10^{-3}$ which is reduced by a factor of $.1$ after 100 and 200 epochs. The batch size is 240 which implies 1000 iterations per epoch.
	
	Given the 20 networks of different, we split the loss into the non-Gaussianity $\Gg$ and non-Standardness $\Ss$ as suggested by \cref{prop:pythagorean-identity}. To do so, we compute the empirical covariances $\Sigma_l$ of 10'000 test samples pushed through each flow.
	
	To relate this experiment to our theory, we take the covariance matrices obtained using the trained flows $\Sigma_l$ and apply \cref{thm:single-layer-precise} on each. This yields an upper bound on the expected non-Standardness after training another network with depth increased by one. In other words, given $\Sigma_l$, \cref{thm:single-layer-precise} predicts an upper bound on the expected $\EE_{Q_{l+1} \sim p(Q)}[\Ss(\Sigma_{l+1})]$.
	We observe that the experimentally observed non-Standardness behaves similar to the upper bound.
	We do not expect this to be the case in general: There might be a trade-off between reducing $\Ss$ and $\Gg$, so the optimization might actually decide for reducing $\Gg$ at the cost of increasing $\Ss$. We only show that with the covariance in \cref{prop:single-layer-whitening}, $\Gg$ does not increase. On the other end, an affine flow might actually be able to reduce the non-Standardness stronger than predicted, as our theory does not take potentially useful cooperation between layers into account.
	
	We average all results over eight runs per depth (i.e.~$8 \cdot 20 = 160$ networks in total). Despite different random orientations in each run, the results are very concentrated: We find error bars so small that they are not visible in \cref{fig:split-loss-convergence}.
	
	We observe that after three blocks, the non-Standardness is close to zero. Here, the flow consists of one convolutional and two fully connected coupling blocks. This justifies the use of convolutional networks for $s$ and $t$; it could have happened that the convolutional layer, being limited to reduce correlations between pixels only locally, does not reduce the non-Standardness as strongly as predicted. This justifies the use of convolutional layers for the remaining blocks.
	
	\cref{fig:split-loss-convergence-samples} shows samples from one networks trained for each depth (sampling temperature $T = 0.7$).

	\begin{figure}
		\centering
		\includegraphics[width=\linewidth]{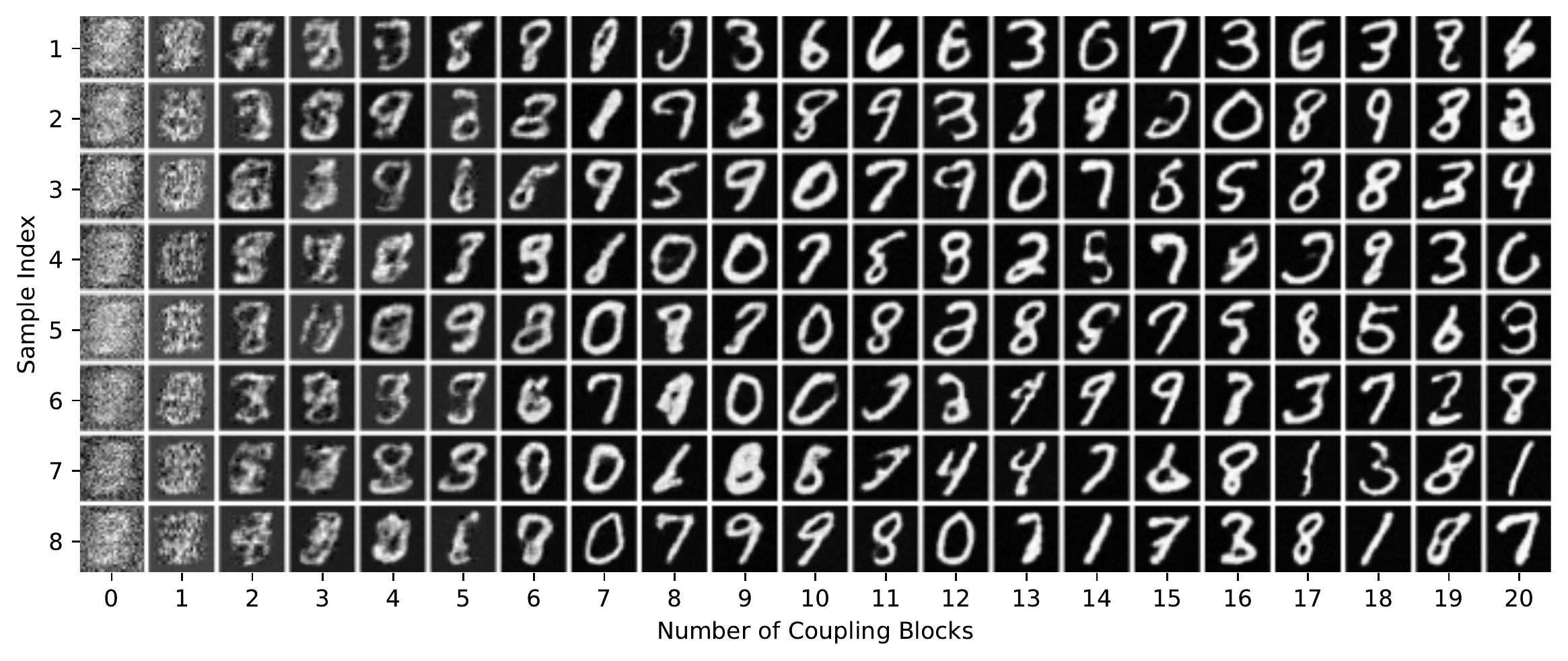}
		\caption{Samples generated by the affine coupling flows with varying depth trained for \cref{fig:split-loss-convergence}. Each column shows eight samples by a network of the corresponding depth.}
		\label{fig:split-loss-convergence-samples}
	\end{figure}

	\subsection{Single layer on EMNIST digit covariance}
	\label{app:exp-covariance}
	
	This experiment confirms that the covariance minimizing the non-Standardness $\Ss(\Sigma_1)$ after a single layer is correctly predicted by \cref{prop:single-layer-whitening}.
	
	To get an interesting covariance matrix, we flatten the EMNIST digits training data and compute its covariance matrix $\Sigma$, as depicted in \cref{fig:emnist-single-layer-covariance} on the left.
	We then sample a multivariate Gaussian with this covariance matrix and train a single affine coupling layer. As the data is Gaussian, we can train with the standard maximum likelihood loss as it is equivalent to the non-Standardness $\Ss$. We use Adam with a learning rate $0.05$, a batch size of $2048$ and train for $512$ iterations.

	\subsection{Single block on toy data}
	\label{app:exp-single-layer}
	
	This experiment explores the average non-Standardness that can be reached by a single layer by modifying the covariance as given by \cref{prop:single-layer-whitening}. It also aims to confirm the upper bounds shown in \cref{thm:single-layer-precise,thm:single-layer-guarantee}.
	
	We build a family of toy covariance matrices to work with. As the data will be randomly rotated anyway, we choose the matrices to be diagonal w.l.o.g., i.e.~we directly design the eigenvalue spectrum of each covariance. We prescribe this spectrum by a continuous function $\mu: [0, 1] \to \RR_+$. It is chosen bijective to ensure that the eigenvalues are distinct. We then define the eigenvalues as follows:
	\eql{
		\mu_i = \mu\big(\tfrac{i}{D-1}\big) \quad i = 0, \dots, D-1.
	}
	With this approach, we can systematically modify eigenvalue/noise spectra.
	
	Given a vector of eigenvalues $(\mu_i)_i$, we need to ensure that its mean is one. We do so by dividing by the mean:
	\eql{
		\nu_i := \frac{\mu_i}{\sum_{i=1}^D \mu_i / D}.
	}
	
	Finally, we add a scaling parameter $s > 0$ that defines how close the spectrum is to the identity:
	\eql{
		\lambda_i^{(s)} := (\nu_i - 1) \cdot s + 1.
	}
	The non-Standardness strictly decreases as $s$ comes closer to $0$.
	As the eigenvalues always have to be positive, $s$ must be chosen smaller than: \eql{
		s < \frac{1}{1 - \lambda_{\min}} =: s_{\max}.
	}
	
	Given a spectrum $\lambda_i^{(s)}$, we build a diagonal covariance matrix
	\eql{
		\Sigma = \Diag(\lambda_i^{(s)})_{i=1}^D.
	}
	
	For the experimental baseline, we sample $N_\text{rot}$ orthogonal and unitary rotation matrices $Q \sim p(Q)$ from the corresponding Haar measure over $O(D)$ and $U(D)$. We employ \texttt{scipy.stats.ortho\_group} respectively \texttt{scipy.stats.unitary\_group}.
	This yields the covariance of the rotated data:
	\eql{
		\Sigma_0 = Q \Sigma Q^T.
	}
	(Or, $Q^*$ instead of $Q^T$ if we average over unitary matrices).
	
	We do not train affine coupling layers directly. Instead, we make use of the single layer output covariance $\Sigma_1$ from \cref{prop:single-layer-whitening}.
	
	We choose the following numerical values for $s$: To get a close look at the case where $s \to 0$ and correspondingly $\Ss \to 0$, we take $N_\text{scale} / 3$ geometrically spaced points in $[0.001s_{\max}, 0.9s_{\max}]$. To accurately capture the off-minimum behavior, we add to that $2N_\text{scale} / 3$ linearly spaced points between $[0.9s_{\max}, .999s_{\max}]$.
	
	We choose $N_\text{rot} = 100$ and $N_\text{scale} = 150$ for all experiments.
	To save computational resources, we re-use the rotations sampled for the first scale for the remaining.
	
	In \cref{fig:single-layer-experiment}, we showed the experiment for the parameterized spectra $\mu(x) = x^2$ and $\mu(x) = x^8$. For both, \cref{fig:bound-comparison-scaling-example} shows which rescaled eigenvalue spectra were used in this experiment.
	In \cref{fig:bound-comparison-more}, we give examples for more spectra.
	
	\begin{figure}
		\centering
		\includegraphics[width=\linewidth]{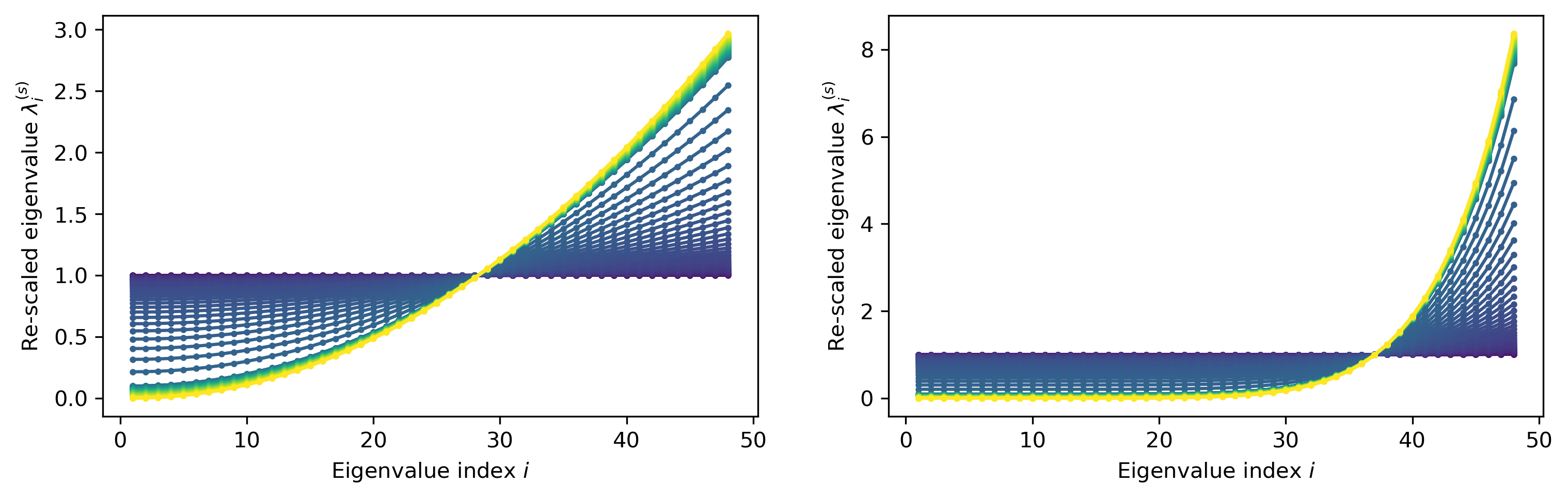}
		\caption{Eigenvalue spectra used for experiment depicted in \cref{fig:single-layer-experiment}. \textit{(Left)} $\mu(x) = x^2$ and \textit{(right)} $\mu(x) = x^8$. Each line corresponds to a different scaling $s$.}
		\label{fig:bound-comparison-scaling-example}
	\end{figure}
	
	\begin{figure}
		\centering
		\includegraphics[width=\linewidth]{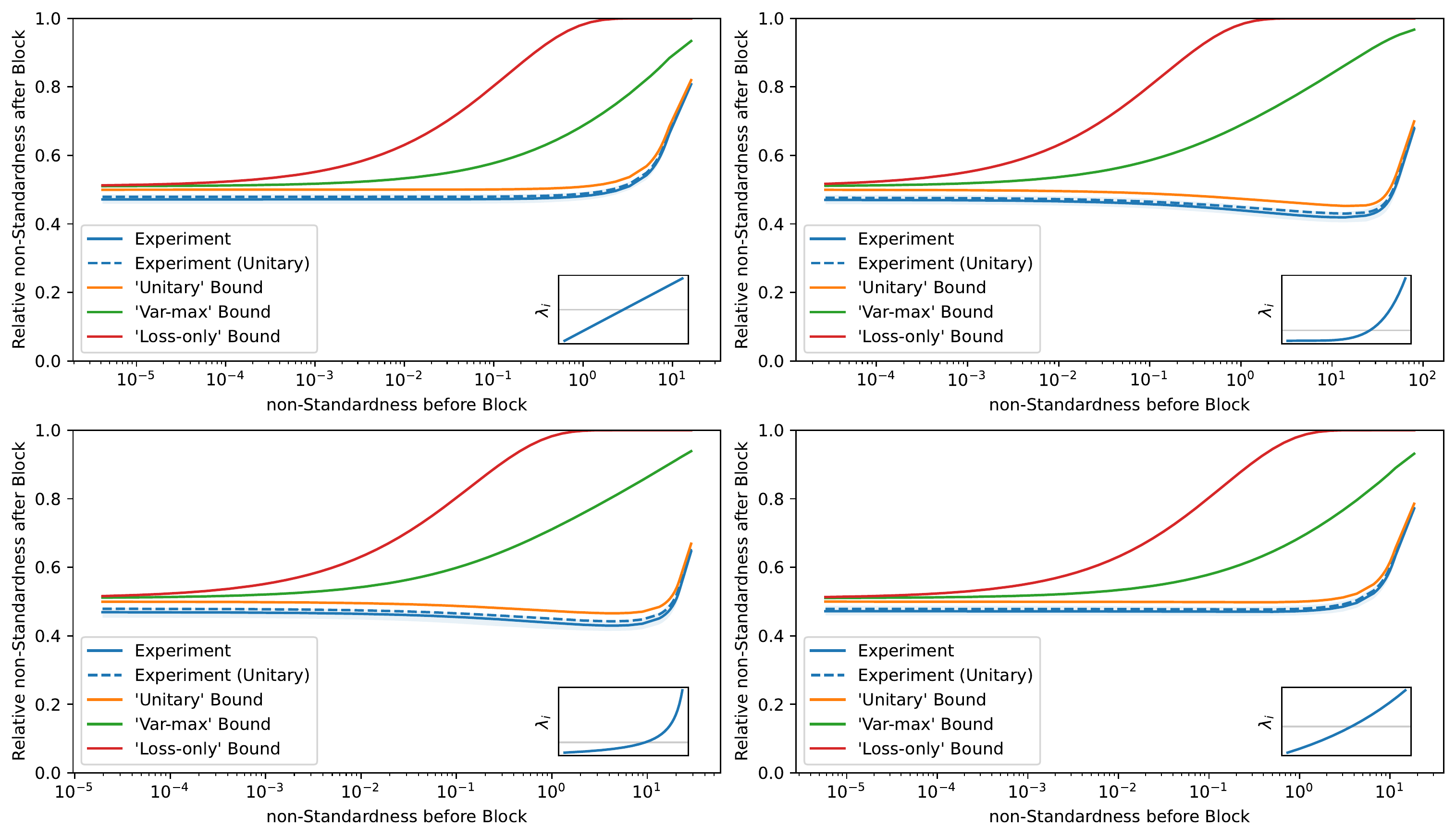}
		\caption{Examples for single layer relative non-Standardness on more eigenvalue spectra: \textit{(Top left)} $\mu(x) = x$, \textit{(top right)} $\mu(x) = x^5$, \textit{(bottom left)} $\mu(x) = \frac{1}{1.1 - x}$, \textit{(bottom right)} $\mu(x) = \exp(x)$. More details in \cref{fig:single-layer-experiment}.}
		\label{fig:bound-comparison-more}
	\end{figure}

	\subsection{Layer-wise training on toy data}
	\label{app:exp-multi-layer}
	
	In this experiment, we track the non-Standardness as layers are added, check \cref{thm:deep-network-guarantee}, and compare the convergence rate \cref{eq:convergence-rate} to the other bounds in \cref{thm:single-layer-precise,thm:single-layer-guarantee}.
	
	This experiment uses a different set of toy covariances than \cref{app:exp-single-layer}. This time, we build a plethora of different initial covariances (eigenvalue spectra) that include extreme cases:
	\begin{enumerate}
		\item All eigenvalues are set to 1 except for one that is varying.
		\item All eigenvalues have the same value that is varied, except for one that is set to 1.
		\item Split the eigenvalues into two halves, respectively having the same value: The first half is varied, the second half assume the inverse value of the first half.
		\item Randomly sample all eigenvalues uniformly from $[0, 2]$.
		\item Randomly sample all eigenvalues between such that the logarithm is uniformly distributed over $[1/v_{\max}, v_{\max}]$.
	\end{enumerate}
	Whenever we vary the value of any eigenvalue, we take $N_\text{vary}$ scalars geometrically spaced between $1/v_{\max}$ and $v_{\max}$. We exclude the case where all eigenvalues are equal to $1$, implying a non-Standardness of $0$.
	
	To fulfill \cref{as:non-degenerate-eigenvalues}, we do not actually assign the same value to eigenvalues, but multiply them each with a linearly increasing factor in $(1 - \epsilon, 1 + \epsilon)$. We do not observe any change in experimental behavior from this, but this allows us evaluating \cref{thm:single-layer-precise}.
	
	Given the dataset of eigenvalues, we build diagonal covariances, repeatedly apply random rotations and the whitening procedure in \cref{prop:single-layer-whitening}. The details are given in \cref{alg:multi-layer-experiment}.
	For each input covariance, we obtain $N_\text{rot}$ trajectories of covariances.
	
	\begin{algorithm}
		\caption{Multi-layer non-Standardness experiment}
		\label{alg:multi-layer-experiment}
		\begin{algorithmic}
			\REQUIRE Input covariances $\Sigma^{(i)}, i=1, \dots, N$, number of rotations $N_\text{rot}$, number of layers~$L$.
			\STATE $\Sigma_0^{(i, r)} \gets \Sigma^{(i)}$ \textbf{for} $ i = 1, \dots, N; r = 1, \dots, N_\text{rot}$ \COMMENT{Copy each input covariance $N_\text{rot}$ times}
			\FOR{$l = 1, \dots, L$}
			\STATE $Q^{(r)} \sim O(D) $ \textbf{for} $ r = 1, \dots, N_\text{rot}$ \COMMENT{Sample rotations}
			\STATE $\Sigma_{l-1}^{(i, r)}{}'  \gets Q^{(r)} \Sigma_{l-1}^{(i, r)} (Q^{(r)})^\transy$ \textbf{for} $ i = 1, \dots, N; r = 1, \dots, N_\text{rot}$ \COMMENT{Apply rotations}
			\STATE $\Sigma_{l}^{(i, r)} \gets $ \cref{prop:single-layer-whitening} on $\Sigma_{l-1}^{(i, r)}{}'$ \COMMENT{Apply whitening step}
			\ENDFOR
			\ENSURE $\{ \Sigma_{l}^{(i, r)} \}_{l=1}^L $ \textbf{for} $ i = 1, \dots, N; r = 1, \dots, N_\text{rot}$.
		\end{algorithmic}
	\end{algorithm}
	
	We evaluate the non-Standardness of each covariance matrix $\Ss(\Sigma_l^{(i, r)})$ and average over rotations. This is shown in the left plot in \cref{fig:multi-layer}.
	
	In addition, we compute the relative non-Standardness between layers:
	\eql{
		\Ss(\Sigma^{(i, r)}_{l}) / \Ss(\Sigma^{(i, r)}_{l - 1}).
	}
	This quantity is averaged over rotations $r$ and instances $i$. It is depicted together with the corresponding interquartile range (IQR) in the right half of \cref{fig:multi-layer}.
	
	We also evaluate each of the bounds on $\EE_Q[\Ss(\Sigma^{(i, r)}_{l+1})]$ in \cref{thm:single-layer-precise,thm:single-layer-guarantee} given $\Sigma^{(i, r)}_{l}$ and divide it by the non-Standardness $\Ss(\Sigma^{(i, r)}_{l})$. Again, we average over rotations and iterations.
	
	Averaging over rotations might be counter-intuitive as the bounds explicitly calculate a value that is an average: It is necessary because for each initial covariance, we have $N_\text{rot}$ trajectories with different convergence behavior.
	Let us make this explicit. Denote by $B$ any of the bounds in \cref{thm:single-layer-precise,thm:single-layer-guarantee}:
	\eql{
		\frac{\EE_{Q_{l+1} \sim p(Q)}[\Ss(\Sigma_{l+1}^{(i, r)}(Q_{l+1}))]}{\Ss(\Sigma_{l}^{(i, r)})} \leq \frac{B(\Sigma_{l}^{(i, r)})}{\Ss(\Sigma_{l}^{(i, r)})}.
	}
	We average the quantity on the right over the different trajectories, i.e.~over $i, r$. It only depends on the covariances in the $l$th layer in contrast to the expression on the left.
	
	As hyperparameters to the experiment, we choose $D=48, L=32, N_\text{vary}=128, N_\text{rot} = 32, v_{\max} = 1000, \epsilon = 10^{-5}$. We stop a trajectory once the non-Standardness falls below $10^{-9}$ to avoid numerical instabilities.

	\section{Detailed proofs}
	\label{app:proofs}

	\subsection{Proof of \cref{prop:pythagorean-identity}}
	\label{app:pythagorean-identity-proof}
	
	The explicit form of the non-Standardness is given by the KL divergence between the two multivariate Gaussians $\Nn(\mean, \Sigma)$ and $\Nn(0, I)$:
	\eqall{&
		\KL{\Nn(\mean, \Sigma)}{\Nn(0, I)} \\&
		= \EE_{x \sim \Nn(\mean, \Sigma)}[\log \Nn(x; \mean, \Sigma) - \log \Nn(x; 0, I)] \\&
		= \EE_{x \sim \Nn(\mean, \Sigma)}[-\tfrac12 \log \det(2 \pi \Sigma) - \tfrac12 (x - \mean)^\transy \Sigma^{-1} (x - \mean) + \tfrac12 \log \det(2 \pi I_D) + \tfrac12 \norm{x}^2] \\&
		= \tfrac12 \left( -\log \det(\Sigma) + \EE_{x \sim \Nn(\mean, \Sigma)}[ - \tfrac12 (x - \mean)^\transy \Sigma^{-1} (x - \mean) + \tfrac12 \norm{x}^2] \right) \\&
		= \tfrac12 ( \norm{\mean}^2 + \tr\Sigma - D - \log\det\Sigma ).
	}{\label{eq:kl-gaussians}}
	
	\begin{proof}
		We start with the first decomposition in \cref{eq:gauss-standard-split}.
		\eqal{&
			\KL{p}{\Nn(0, I)} - \KL{p}{\Nn(\mean, \Sigma)} \\&
			= \EE_{x \sim p(x)}[\log p(x) - \log \Nn(x; 0, I) - \log p(x) + \log \Nn(x; \mean, \Sigma)] \\&
			= \EE_{x \sim p(x)}[- \log \Nn(x; 0, I) + \log \Nn(x; \mean, \Sigma)] \\&
			= \frac12 \EE_{x \sim p(x)}[D\log(2 \pi) + \norm{x}^2 - D\log(2 \pi) - \log\det\Sigma - (x - \mean)^\transy \Sigma^{-1} (x - \mean)] \\&
			= \frac12 \EE_{x \sim p(x)}[\norm{x}^2 - \log\det\Sigma - (x - \mean)^\transy \Sigma^{-1} (x - \mean)] \\&
			= \frac12 \left(\EE_{x \sim p(x)}[\norm{x}^2] - \log\det\Sigma - \EE_{x \sim p(x)}[(x - \mean)^\transy \Sigma^{-1} (x - \mean)] \right).
		}
		
		The open expectation values read:
		\eql{
			\EE_{x \sim p(x)}[\norm{x}^2]
			= \EE_{x \sim p(x)}[\sum_{i=1}^D x_i^2] 
			= \sum_{i=1}^D \EE_{x \sim p(x)}[x_i^2]
			= \sum_{i=1}^D (\mean_i^2 + \Sigma_{ii})
			= \norm{\mean}^2 + \tr\Sigma,
		}
		and interpreting $(x - \mean)$ as a $\RR^{D \times 1}$ matrix, we can re-write using the trace. Then use the cyclic property and linearity of the trace:
		\eqal{
			\EE_{x \sim p(x)}[(x - \mean)^\transy \Sigma^{-1} (x - \mean)]  &
			= \EE_{x \sim p(x)}[\tr((x - \mean)^\transy \Sigma^{-1} (x - \mean))]  \\&
			= \EE_{x \sim p(x)}[\tr((x - \mean) (x - \mean)^\transy \Sigma^{-1})]  \\&
			= \tr(\EE_{x \sim p(x)}[(x - \mean) (x - \mean)^\transy] \Sigma^{-1})  \\&
			= \tr(\Sigma \Sigma^{-1})  \\&
			= D.
		}
		
		Inserting the two expectation values, we identify:
		\eqal{
			\KL{p}{\Nn(0, I)} - \KL{p}{\Nn(\mean, \Sigma)}  &
			= \frac12 \left(\norm{\mean}^2 + \tr\Sigma - \log\det\Sigma - D \right)  \\&
			= \KL{\Nn(\mean, \Sigma)}{\Nn(0, I)},
		}
		and obtain \cref{eq:gauss-standard-split}.
		
		Now we move on to show \cref{eq:diagonal-standard-split}:
		\eqal{
			\Cc(p) &
			= \KL{\Nn(\mean, \Sigma)}{\Nn(\mean, \Diag(\Sigma))} \\&
			= \frac12 \left(\tr\left((\Diag \Sigma)^{-1} \Sigma\right) - D + \log\frac{\det(\Diag(\Sigma))}{\det \Sigma} \right) \\&
			= \frac12 \log\frac{\det(\Diag(\Sigma))}{\det \Sigma},
		}
		and
		\eqal{
			\KL{\Nn(\mean, \Diag(\Sigma))}{\Nn(0, I)} &
			= \frac12 \left( \tr \Diag(\Sigma) - D - \log(\det(\Diag(\Sigma))) \right) \\&
			= \frac12 \left( \tr \Sigma - D - \log(\det(\Diag(\Sigma))) \right).
		}
		Adding the two divergences yields \cref{eq:diagonal-standard-split}.
	\end{proof}

	\subsection{Proof of \cref{prop:single-layer-whitening}}
	\label{app:single-layer-whitening-proof}
	
	We first show that an affine-linear function $g(x)$ as assumed in \cref{prop:single-layer-whitening} cannot change the non-Gaussianity $\Gg$:
	\begin{lemma}
		\label{lem:linear-constant-non-gaussianity}
		Given a $D$-dimensional distribution and an affine-linear function
		\eql{
			g(x) = Ax + b
		}
		for some $A \in \RR^{D \times D}$ with $\det A > 0$ and $b \in \RR^D$.
		Then:
		\eql{
			\Gg(g_\sharp p) = \Gg(p).
		}
	\end{lemma}
	\begin{proof}
		The non-Gaussianity $\Gg$ is given by:
		\eql{
			\Gg(p) = \KL{p(x)}{\Nn(\mean, \Sigma)}.
		}
		Mean and covariance of the push-forward of $p$ via $g$ read:
		\eqal{
			\EE_{x \sim p(x)}[g(x)] &
			= \EE_{x \sim p(x)}[Ax + b]
			= Am + b = m_1, \\
			\Cov_{x \sim p(x)}[g(x)] &
			= \Cov_{x \sim p(x)}[Ax + b]
			= A \Sigma A^\transy = \Sigma_1.
		}
		Thus, the non-Gaussianity after applying $g$ reads:
		\eql{
			\Gg(g_\sharp p) = \KL{g_\sharp p}{\Nn(\mean_1, \Sigma_1)}.
		}
		The push-forward of $\Nn(\mean, \Sigma)$ via $g$ is identical to the normal distribution that occurs in the non-Gaussianity of $g_\sharp p$:
		\eql{
			g_\sharp \Nn(\mean, \Sigma) = \Nn(\mean_1, \Sigma_1),
		}
		Now, we make use of the fact that the KL divergence is invariant if both arguments are transformed by any invertible function $g$:
		\eql{
			\KL{p_1(x)}{p_2(x)} = \KL{(g_\sharp p_1)(x)}{(g_\sharp p_2)(x)}.
		}
		Together,
		\eql{
			\Gg(g_\sharp p) = \Gg(p).
		}
	\end{proof}

	We now turn to the \textbf{proof of \cref{prop:single-layer-whitening}}:
	\begin{proof}
		We aim to find the affine-linear coupling layer $f_\text{cpl}$ minimizing $\Ss(\Sigma_1)$.
		By \cref{lem:linear-constant-non-gaussianity}, $\Gg$ does not change.
		
		The affine-linear coupling $f_\text{cpl}$ has the following form:
		\eql{
			x_1 = \begin{pmatrix}
				\Diag(r) & 0 \\
				T & \Diag(s)
			\end{pmatrix}
			\begin{pmatrix}
				p_0 \\
				a_0
			\end{pmatrix}
			+
			\begin{pmatrix}
				u \\
				t
			\end{pmatrix}
			=: A x_0 + b.
		}
		To make the coupling affine-linear, $r, s \in \RR_+^{D/2}$ are positive vectors, $u, t \in \RR^{D/2}$ are vectors and $T \in \RR^{D/2 \times D/2}$ is the matrix describing the linear dependence of $a_1$ on $p_0$.
		
		By linearity of expectation, the mean of $x_1$ reads:
		\eql{
			\mean_1 = A \mean_0 + b.
		}
		Write $S := \Diag(s)$ and $R := \Diag(r)$ so that the covariance of $x_1$ is given by:
		\eqal{
			\Sigma_1 &
			:= \Cov[x_1] 
			= A \Sigma_0 A^\transy \\&
			= \begin{pmatrix}
				R & 0 \\
				T & S
			\end{pmatrix} \begin{pmatrix}
				\Sigma_{0, pp} & \Sigma_{0, pa} \\
				\Sigma_{0, ap} & \Sigma_{0, aa}
			\end{pmatrix} \begin{pmatrix}
				R & T^\transy \\
				0 & S
			\end{pmatrix} \\&
			= \begin{pmatrix}
				R \Sigma_{0,pp} R & R (\Sigma_{0,pa} S + \Sigma_{0,pp} T^\transy) \\
				(T \Sigma_{0,pp} + S \Sigma_{0,ap}) R & (T \Sigma_{0,pa} + S \Sigma_{0,aa}) S + (T \Sigma_{0,pp} + S \Sigma_{0,ap}) T^\transy
			\end{pmatrix}
			\label{eq:cov-affine-linear-layer}
		}
		Together, the non-Standardness of $x_1$ is given from \cref{eq:non-standardness}
		\eqall{&
			\Ss(\mean_1, \Sigma_1) \\&
			= \tfrac12 \Big( \norm{\mean_1}^2 + \tr\Sigma_1 - D - \log\det\Sigma_1 \Big) \\&
			= \tfrac12 \Big( \norm{A \mean_0 + b}^2 + \tr(R^2 \Sigma_{0,pp}) + \tr(T\Sigma_{0,pa} S) + \tr(S^2 \Sigma_{0,aa}) + \tr(T\Sigma_{0,pp}T^\transy)  \\&\qquad\qquad
			+ \tr(S\Sigma_{0,ap}T^\transy)
			- D - \log\det\Sigma_0 - \log\det R - \log\det S \Big).
		}{\label{eq:non-standardness-after-explicit}}
		
		To find the minimum of $\Ss(\mean_1, \Sigma_1)$, minimize the above over $r, s, T$ and $b$:
		\eqal{
			\argmin_{r, s, T, b} \Ss(\mean_1, \Sigma_1).
		}
		It is easy to see that $b = -A \mean_0$ minimizes \cref{eq:non-standardness-after-explicit} as in this case $\mean_1 = 0$.
		
		At the minimum, we find for $r$:
		\eql{
			0 = \frac{\partial \Ss(\mean_1, \Sigma_1)}{\partial r_m} = -\frac1{r_m} + r_m (\Sigma_{0,pp})_{mm},
		}
		for some $m = 1, \dots, D/2$.
		We read off that $r_m = (\Sigma_{0,pp})_{mm}^{-1/2}$. In matrix notation:
		\eql{
			R = \Diag(\Sigma_{0,pp})^{-1/2}.
		}
		
		For $s, T$, we find the system:
		\eqal{
			0 &= \frac{\partial \Ss(\mean_1, \Sigma_1)}{\partial s_n}
			= -\frac1{s_n} + \sum_{j=1}^{D/2} T_{nj} (\Sigma_{0,pa})_{jn} + s_n (\Sigma_{0,ap})_{nn} \\
			0 &= \frac{\partial \Ss(\mean_1, \Sigma_1)}{\partial T_{op}} 
			= s_p (\Sigma_{0,pa})_{po} + \sum_{k=1}^{D/2}  T_{pk} (\Sigma_{0,aa})_{ko}.
		}
		Multiplying the first equation by $s_n$, we find in matrix notation:
		\eqal{
			\one &= \Diag(T \Sigma_{0,pa} S + S^2 \Sigma_{0,aa}) \\
			0 &= S \Sigma_{0,pa} + T \Sigma_{0,pp}.
		}
		We solve the second equation for $T$ (we use that $\Sigma_{0,pp}$ is invertible as it is positive definite):
		\eql{
			T = - S \Sigma_{0,ap} \Sigma_{0,pp}^{-1},
		}
		and insert into the first:
		\eqal{
			\one  &
			= \Diag(- S \Sigma_{0,ap} \Sigma_{0,pp}^{-1} \Sigma_{0,pa} S + S^2 \Sigma_{0,aa}) \\&
			= \Diag(- S^2 \Sigma_{0,ap} \Sigma_{0,pp}^{-1} \Sigma_{0,pa} + S^2 \Sigma_{0,aa}).
		}
		The last step is due to $\Diag(\cdot)$ linear and $S$ diagonal. We read off that:
		\eql{
			S = \Diag(\Sigma_{0,aa} - \Sigma_{0,ap} \Sigma_{0,pp}^{-1} \Sigma_{0,pa})^{-1/2}.
		}
		Alternative solutions with negative signs are discarded by convention (without an effect on the covariance).
		
		Inserting into \cref{eq:cov-affine-linear-layer}, we find:
		\eqal{
			\Sigma_{1,pp} &= R \Sigma_{0,pp} R = \Diag(\Sigma_{0,pp})^{-1/2} \Sigma_{0,pp} \Diag(\Sigma_{0,pp})^{-1/2}
			= M(\Sigma_{0,pp}), \\
			\Sigma_{1,pa} &= R (\Sigma_{0,pa} S + \Sigma_{0,pp} T^\transy)
			= R (\Sigma_{0,pa} S - \Sigma_{0,pp} \Sigma_{0,pp}^{-1} \Sigma_{0,pa} S)
			= 0, \\
			\Sigma_{1,ap} &= \Sigma_{1,pa}^\transy = 0, \\
			\Sigma_{1,aa} &= (T \Sigma_{0,pa} + S \Sigma_{0,aa}) S + (T \Sigma_{0,pp} + S \Sigma_{0,ap}) T^\transy \\&
			= (-S \Sigma_{0,ap} \Sigma_{0,pp}^{-1} \Sigma_{0,pa} + S \Sigma_{0,aa}) S + (S \Sigma_{0,ap} \Sigma_{0,pp}^{-1} \Sigma_{0,pp} - S \Sigma_{0,ap}) \Sigma_{0,pp}^{-1} \Sigma_{0,pa} S \\&
			= S (\Sigma_{0,aa} - \Sigma_{0,ap} \Sigma_{0,pp}^{-1} \Sigma_{0,pa}) S
			= M(\Sigma_{0,aa} - \Sigma_{0,ap} \Sigma_{0,pp}^{-1} \Sigma_{0,pa}).
		}
		This concludes the proof:    
		\eql{
			\mean_1 = 0, \qquad \Sigma_1
			= \begin{pmatrix}
				M(\Sigma_{0,pp}) & 0 \\
				0 & M(\Sigma_{0,aa} - \Sigma_{ap}\Sigma_{pp}^{-1}\Sigma_{pa})
			\end{pmatrix}.
		}
	\end{proof}

	\subsection{Proof of \cref{thm:single-layer-precise}}
	\label{app:single-layer-precise-proof}
	
	The following statement will help us along the way:
	\begin{lemma}
		\label{lem:helper-identities}
		For $A \in \CC^D$, $Q \in \{ O(D), U(D) \}$ with the corresponding Haar measure $p(Q)$:
		\eql{
			\EE_{Q \in p(Q)}[(Q A Q^*)_{ii}] = \frac1D \tr(A).
		}
	\end{lemma}
	\begin{proof}
		By symmetry, $\EE_{Q}[(Q A Q^*)_{11}] = \EE_{Q}[(Q A Q^*)_{ii}] $ for $i = 1, \dots, D$. Thus, $\EE_{Q}[(Q A Q^*)_{11}] = \frac1D \sum_{i=1}^D \EE_{Q}[(Q A Q^*)_{ii}] = \frac1D \EE_{Q}[\tr(Q A Q^*)] = \frac1D \tr A$.
	\end{proof}
	When we write $Q^*$, we mean conjugate transpose if $Q$ is sampled from the unitary group $U(D)$, and transpose if $Q$ is from the orthogonal group $O(D)$. Whenever we only consider orthogonal $Q$, we will resort back to writing $Q^\transy$.
	
	This allows us to directly estimate $\EE_{Q \sim p(Q)}[\log \det M_p^2]$:
	
	\begin{lemma}
		\label{lem:passive-scaling-contribution}
		With the definitions in \cref{sec:single-block-guarantees}, $p(Q)$ either the Haar measure of orthogonal or unitary matrices, and \cref{as:normalized-covariance}. Then:
		\eql{
			\EE_{Q \sim p(Q)}[\log \det M_p^2] \geq 0.
		}
	\end{lemma}
	\begin{proof}
		$M_p$ is given by:
		\eql{
			M_p^2 = \Diag(\Sigma_{0,pp})^{-1}.
		}
		The corresponding expectation value can be estimated via Jensen's inequality:
		\eqal{&
			\EE_{Q \sim p(Q)}[\log \det M_p^2]
			= \EE_{Q \sim p(Q)}[\log \det \Diag(\Sigma_{0,pp})^{-1}] \\&
			= -\EE_{Q \sim p(Q)}[\log \prod_{i=1}^{D/2} (\Sigma_{0,pp})_{ii}]
			= -\sum_{i=1}^{D/2} \EE_{Q \sim p(Q)}[\log (\Sigma_{0,pp})_{ii}] \\&
			\geq -\sum_{i=1}^{D/2} \log \EE_{Q \sim p(Q)}[(\Sigma_{0,pp})_{ii}] 
			= - \frac{D}{2} \log \tr\Sigma / D \\&
			= 0.
		}
		By \cref{as:normalized-covariance}, $\tr\Sigma = D$.
		We have used \cref{lem:helper-identities} for evaluating $\EE_{Q \sim p(Q)}[(\Sigma_{0,pp})_{ii}]$.
	\end{proof}
	
	As mentioned in \cref{sec:single-block-guarantees}, the main difficulty in estimating $\EE_{Q \sim p(Q)}[\Ss(\Sigma_1(Q))]$ lies in $\EE_{Q \sim p(Q)}[\log \det M_a^2]$. The following subsections show a path to do so.

	\subsubsection{Problem reformulation}
	
	In a first step, we reformulate this expectation so that it can be computed with the help of projected orbital measures \cite{olshanski_projections_2013}.
	
	We split the expectation over the Haar measure $p(Q)$ in two parts: One that defines which eigenvalues the $(D/2) \times (D/2)$ block $\Sigma_{0,aa}$ has (denote this as $Q_{ap}$) and, conditioned on this, another which rotates $\Sigma_{0,aa}$ into all possibles bases (denote this as $Q_a$). Formally, write $Q$ as:
	\eql{
		Q = \begin{pmatrix}
			I & 0 \\
			0 & Q_a
		\end{pmatrix} Q_{ap}.
	}
	
	We will replace the Schur complement $\Sigma_{0,aa} - \Sigma_{0,ap}\Sigma_{0,pp}^{-1}\Sigma_{0,pa}$ appearing in \cref{prop:single-layer-whitening} by the corresponding block of the precision matrix $P_0 := \Sigma_0^{-1} = (Q \Sigma^{-1} Q^*)^{-1} = Q \Sigma^{-1} Q^*$ (e.g.~\cite[Section (0.7.3)]{horn_matrix_2012}:
	\eql{
		(P_{0,aa})^{-1} = ((\Sigma^{-1}_0)_{aa})^{-1} = \Sigma_{0,aa} - \Sigma_{0,ap}\Sigma_{0,pp}^{-1}\Sigma_{0,pa}.
	}
	We give more details in the proof of the following lemma, which formalizes this step:
	
	\begin{lemma}
		\label{lem:active-scaling-by-precision-eigenvalues}
		Given the definitions in \cref{sec:single-block-guarantees} and \cref{as:normalized-covariance}. It holds that
		\eql{
			\EE_{Q \sim p(Q)}[\log \det M_a^2] \geq - \sum_{i=1}^{D/2} \log \EE_{Q_a \sim p(Q_a|Q_{ap})}[((P_{0,aa})^{-1})_{ii}].
		}
	\end{lemma}
	\begin{proof}
		By \cref{prop:single-layer-whitening}, $M_a^2$ is given by:
		\eql{
			M_a^2 = \Diag(\Sigma_{0,aa} - \Sigma_{0,ap}\Sigma_{0,pp}^{-1}\Sigma_{0,pa})^{-1}.
		}
		Being diagonal, its determinant is given by the product of its diagonal entries:
		\eqal{
			\EE_{Q \sim p(Q)}[\log \det M_a^2] &
			= \EE_{Q \sim p(Q)}[\log \prod_{i=1}^{D/2} (\Sigma_{0,aa} - \Sigma_{0,ap}\Sigma_{0,pp}^{-1}\Sigma_{0,pa})_{ii}^{-1}] \\&
			= \sum_{i=1}^{D/2} \EE_{Q \sim p(Q)}[\log ((\Sigma_{0,aa} - \Sigma_{0,ap}\Sigma_{0,pp}^{-1}\Sigma_{0,pa})_{ii}^{-1})] \\&
			= -\sum_{i=1}^{D/2} \EE_{Q \sim p(Q)}[\log ((\Sigma_{0,aa} - \Sigma_{0,ap}\Sigma_{0,pp}^{-1}\Sigma_{0,pa})_{ii})].
		}
		Evaluating this expression is hard mainly because the $\Sigma_{0,ap}\Sigma_{0,pp}^{-1}\Sigma_{0,pa}$ involves the inverse of $\Sigma_{0,pp} = (Q \Sigma Q^*)_{pp}$, which depends on $Q$.
		
		To circumvent this, note the following property of any nonsingular matrix $M$ \cite[Section (0.7.3)]{horn_matrix_2012}. Split $M$ into blocks as:
		\eql{
			M = \begin{pmatrix}
				A & B \\
				B^* & C
			\end{pmatrix},
		}
		and do the same for its inverse:
		\eql{
			M^{-1} = \begin{pmatrix}
				A' & B' \\
				B'^* & C'
			\end{pmatrix}
		}
		Then, $(A')^{-1} = A - B C^{-1} B^*$, which is called the Schur complement $M / C$. This means we can rewrite
		\eql{
			\Sigma_{0,aa} - \Sigma_{0,ap}\Sigma_{0,pp}^{-1}\Sigma_{0,pa} = (P_{0,aa})^{-1},
		}
		where $P_0 = \Sigma_0^{-1}$ is the \textit{precision matrix} of the rotated data. Given a rotation $Q$, it can easily be obtained from the precision matrix of the data in its original rotation:
		\eql{
			\Sigma_0 = Q \Sigma Q^*, \qquad P_0 = Q \Sigma^{-1} Q^*.
		}
		
		Inserting this, we find the expectation value:
		\eql{
			\EE_{Q \sim p(Q)}[\log \det M_a^2]
			= -\sum_{i=1}^{D/2} \EE_{Q \sim p(Q)}[\log (((P_{0,aa})^{-1})_{ii})].
		}
		The logarithm can be drawn out via Jensen's inequality:
		\eql{
			-\sum_{i=1}^{D/2} \EE_{Q \sim p(Q)}[\log (((P_{0,aa})^{-1})_{ii})]
			\geq -\sum_{i=1}^{D/2} \log(\EE_{Q \sim p(Q)}[((P_{0,aa})^{-1})_{ii}]).
		}
		This concludes the statement.
	\end{proof}

	\subsubsection{Projected orbit expectation}
	
	The theory of projected orbital measures describes the distribution of eigenvalues of a randomly projected submatrix of some given matrix. Let us formalize this:
	
	Fix a diagonal matrix $A = \Diag(a_1, \dots, a_N)$. Then, then the \textit{orbit} of $A$ is defined as:
	\eql{
		\Oo_A := \{ QAQ^* : Q \in U(D) \}.
	}
	(The same definition also exists for orthogonal $Q \in O(D)$, but we keep it to the level we require here).
	All matrices in the orbit of $\Oo_A$ have the same eigenvalues.
	
	The natural measure (probability distribution) on the orbit $\Oo_A$ is given by the image of the Haar measure on the unitary group $U(D)$. This can be thought of as the uniform measure on the group of unitary rotations. We call this measure the \textit{orbital measure}.
	
	We now cut out the $K \times K$ top left corner out of every matrix in $\Oo_A$:
	\eql{
		P_K\Oo_A := \{ P_K Y: Y \in \Oo_A \}.
	}
	We call this the \textit{projected orbit}. The matrix $P_K$ projects a matrix to its upper left corner:
	\eql{
		P_K = (I_K; 0_{K \times (N-K)}).
	}
	The distribution of matrices in the projected orbit $P_K \Oo_A$ induced by the orbital measure is denoted as the \textit{projected orbital measure} $\mu_{A,K}$.
	We are now interested in the eigenvalues of matrices in the projected orbit $P_K\Oo_A$.
	
	Let $\operatorname{spectrum}$ be the function that assigns a matrix $Y \in \CC^{K \times K}$ its eigenvalues $y_1, \dots, y_K$.
	We will make use of a result that gives the distribution of eigenvalues of matrices in the projected orbit $P_K\Oo_A$. This is called the \textit{radial part of the projected orbital measure} and is denoted as $\nu_{A,K}(x_1, \dots, x_K)$.
	\eql{
		\nu_{A,K}(x_1, \dots, x_K) = \PP_{X \sim \mu_{A,K}}[\operatorname{spectrum}(X) = (x_1, \dots, x_K)].
	}
	In other words, $\nu_{A,K}(x_1, \dots, x_K)$ gives the probability density that a random matrix from the projected orbit of $A$ has exactly eigenvalues $(x_1, \dots, x_K)$.
	Its functional form was shown by \cite{olshanski_projections_2013}:
	
	\begin{theorem}[Radial part of projected orbital measure \cite{olshanski_projections_2013}]
		Fix $A = (a_1, \dots, a_D)$ with $a_1 < \dots < a_D$. For any $K = 1, \dots, D - 1$, the density of eigenvalues of 
		\eql{
			\nu_{A,K}(x_1, \dots, x_K) = c_{D,K} \frac{V(x_1, \dots, x_K) \det[M(a_j; x_i, \dots, x_{D-K+i})]_{i,j=1}^K}{\prod_{j-i \geq D - K + 1}(x_j - x_i)}.
		}
		Here, the constant is given by:
		\eql{
			c_{D, K} = \prod_{i=1}^{K - 1} \begin{pmatrix}
				D - K + i \\
				i
			\end{pmatrix},
		}
		and $M(a; y_1, \dots, y_N)$ is the B-spline:
		\eql{
			M(a; y_1, \dots, y_n) := (N - 1) \sum_{i: y_i > a} \frac{(y_i - a)^{n-2}}{\prod_{r: r \neq i} (y_i - y_r)},
		}
		and $V$ is the Vandermonde polynomial:
		\eql{
			V(y_1, \dots, y_n) = \prod_{i < j} (y_j - y_i).
		}
	\end{theorem}
	
	We will make use of the following variant of the Vandermonde determinant where all powers greater or equal to some $k$ are increased by one:
	\begin{lemma}
		\label{lem:vandermonde-loo}
		For all $n \in \NN$, $k = 1, \dots, n - 1$ and distinct $a_i, i = 1, \dots, n$:
		\eql{
			\det \begin{pmatrix}
				1 & \cdots & a_1^{k-1} & a_1^{k+1} & \cdots & a_1^{n} \\
				\vdots && \vdots & \vdots && \vdots \\
				1 & \cdots & a_n^{k-1} & a_n^{k+1} & \cdots & a_n^{n}
			\end{pmatrix}
			= V(a_1, \dots, a_n) e_{n-k}(a_1, \dots, a_n).
		}
		with the elementary symmetric polynomial $e_K$ given by \cref{eq:inverse-volume-elementary-symmetric}.
	\end{lemma}
	
	\begin{lemma}
		\label{lem:expected-projected-inverse-trace}
		Fix $A = (a_1, \dots, a_N)$ with $a_1 < \dots < a_N$. For any $K = 1, \dots, N - 1$, it holds that:
		\eqal{&
			\EE_{a_1, \dots, a_K \sim \nu_{A,K}(x_1, \dots, x_K)}[x_1^{-1} + \dots + x_K^{-1}] \\&
			= (N - K) (-1)^{N-K} \sum_{j = 1}^N a_j^{N-K-1} \log(a_j) R(a_j; a_{\neq j}) e_{K-1}(a_{\neq j}).
		}
		Here, $R$ is defined in \cref{eq:inverse-volume-elementary-symmetric}.
	\end{lemma}
	\begin{proof}
		We use the Andreief identity in the form of \cite[Lemma 2.1]{krishnaiah_recent_1976}
		\eqal{&
			\EE_{x_1, \dots, x_k \sim \nu_{A,K}}[\tfrac1{x_1} + \dots + \tfrac1{x_k}] \\&
			= Z^{-1} \int_{(\RR^K)_+} (\tfrac1{x_1} + \dots + \tfrac1{x_k}) \det(x_i^{j-1}) \det(M(x_i; a_j, \dots, a_{j+N-K})) \d x_1 \cdots \d x_k \\&
			= Z^{-1} \sum_{k=1}^K \det \int_\RR x^{-\delta_{jk}} x^{j-1} M(x; a_i, \dots, a_{i + N - K}) \d x \\&
			= Z^{-1} \det \int_\RR x^{j-1-\delta_{j1}} M(x; a_i, \dots, a_{i + N - K}) \d x \\&
			= Z^{-1} \det \begin{cases}
				\mu_{-1}(a_i, \dots, a_{i + N - K}) & j = 1 \\
				\mu_{j - 1}(a_i, \dots, a_{i + N - K}) & j > 1
			\end{cases}
		}
		where $\mu_k(t_1, \dots, t_n)$ the $k$th moment of the B-spline with knots $t_1, \dots, t_n$:
		\eql{
			\mu_k(t_1, \dots, t_n) = \int x^k M(x; t_1, \dots, t_n) \d x.
		}
		
		We can now make use of the Hermite–Genocchi formula \cite[Proposition 6.3]{faraut_rayleigh_2015}:
		\eql{
			\int f^{(n-1)}(x) M(x; t_1, \dots, t_n) \d x = (n-1)! f[t_1, \dots, t_n],
		}
		so we can rewrite
		\eql{
			\mu_k(t_1, \dots, t_n) = f_k[t_1, \dots, t_n],
		}
		with
		\eqal{
			f_{-1}(x) &= (n-1) x^{n-2} \log x, \\
			f_{k}(x) &= \begin{pmatrix}
				n + k - 1 \\
				k
			\end{pmatrix}^{-1} x^{n + k - 1}.
		}
		
		Together, we find
		\eql{
			\EE_{x_1, \dots, x_k \sim \nu_{A,K}}[\tfrac1{x_1} + \dots + \tfrac1{x_k}]
			= Z^{-1} \det(f_{i - 1 - \delta_{1i}}[a_j, \dots, a_{j + N - K}]).
		}
		The right hand side can be identified with the right hand side of \cite[Proposition 6.4]{faraut_rayleigh_2015}. It is equal to:
		\eqal{&
			Z^{-1} \det(f_i[a_j, \dots, a_{j + N - K}]) \\&
			= Z^{-1} \left(\prod_{0 < j - i \leq N - K} (a_j - a_i)^{-1}\right) \begin{vmatrix}
				1 & \cdots & 1 \\
				a_1 & \cdots & a_N \\
				\vdots & & \vdots \\
				a_1^{N - K - 1} & \cdots & a_N^{N - K - 1} \\
				f_1(a_1) & \cdots & f_1(a_N) \\
				\vdots & & \vdots \\
				f_K(a_1) & \cdots & f_K(a_N)
			\end{vmatrix} \\&
			= Z^{-1} \left(\prod_{0 < j - i \leq N - K} (a_j - a_i)^{-1} \prod_{k=1}^{K} \begin{pmatrix}
				N - K + k \\
				k
			\end{pmatrix}^{-1} \right) (N - K) \begin{vmatrix}
				1 & \cdots & 1 \\
				a_1 & \cdots & a_N \\
				\vdots & & \vdots \\
				a_1^{N - K - 1} & \cdots & a_N^{N - K - 1} \\
				a_1^{N - K - 1} \log a_1 & \cdots & a_N^{N - K - 1} \log a_N \\
				a_1^{N - K + 1} & \cdots & a_N^{N - K + 1} \\
				\vdots & & \vdots \\
				a_1^{N - 1} & \cdots & a_N^{N - 1}
			\end{vmatrix} \label{eq:polynomial-determinant-single} \\&
			=: C_2 \det(M_{ij})
		}
		
		Here, $C_2$ reduces to:
		\eqal{
			C_2
			= 
			\frac{N - K}{V(a_1, \dots a_n)}.
		}
		
		Then, the determinant of $M_{ij}$ reads:
		\eqal{
			\det M_{ij} &
			= \sum_{j=1}^N (-1)^{N - K + 1 + j} a_j^{N-K-1} \log(a_j) V(a_{\neq j}) \sum_{\substack{i_{1} < \dots < i_{K-1} \\ i_{\dots} \neq j}} a_{i_{1}} \cdots a_{i_{K-1}} \\&
			= V(a) (-1)^{N - K} \sum_{j}  a_j^{N-K-1} \log(a_j) R(a_j; a_{\neq j}) e_{K-1}(a_{\neq j}),
		}
		where $R(a_j; a_{\neq j})$ collects all the terms in $V(a)$ that were not contained in $V(a_{\neq j})$ up to sign:
		\eql{
			R(a_j; a_{\neq j}) = \prod_{\substack{i = 1 \\ i \neq j}}^{n} \frac1{a_i - a_j} = (-1)^{j - 1} V(a_{\neq j}) / V(a).
		}
		Note that the sign of $R(a_j; a_{\neq j})$ flips from $j \to j + 1$, so the alternating nature of the above series remains.
		
		Together, the desired expectation value reads:
		\eqal{&
			\EE_{x_1, \dots, x_k \sim \nu_{A,K}}[\tfrac1{x_1} + \dots + \tfrac1{x_k}] \\&
			= (N - K) (-1)^{N-K} \sum_{j = 1}^N a_j^{N-K-1} \log(a_j) R(a_j; a_{\neq j}) e_{K-1}(a_{\neq j}),
		}
		which concludes the proof.
	\end{proof}
	
	We now connect this result to our situation. This paves the path from the reformulation in \cref{lem:active-scaling-by-precision-eigenvalues} to \cref{thm:single-layer-precise}.
	\begin{corollary}
		\label{lem:active-scaling-contribution}
		For the definitions in \cref{sec:single-block-guarantees} and when \cref{as:unitary-rotation,as:non-degenerate-eigenvalues} are fulfilled, it holds that:
		\eql{
			\EE_{Q \sim p(Q)}[\log\det M_a^2] \geq \tfrac{D}2 \log\!\bigg(\!(-1)^{\tfrac{D}{2}+1} \sum_{i=1}^D \lambda_i^{1-\tfrac{D}{2}} \log(\lambda_i) R(\lambda_i^{-1}; \lambda_{\neq i}^{-1}) e_{\tfrac{D}{2}-1}(\lambda_{\neq i}^{-1}) \!\bigg).
		}
	\end{corollary}
	\begin{proof}
		Then,     \cref{lem:helper-identities} tells us how to integrate over $Q_a$.
		\eql{
			\EE_{Q_a \sim p(Q_a|Q_{ap})}[((P_{0,aa})^{-1})_{ii}] = \tr((Q_{ap} P Q_{ap}^*)^{-1}) = \sum_{i=1}^{D/2} a_i(Q_{ap})^{-1}.
		}
		Here, we denote by $a_i(Q_{ap})$ the $i$th eigenvalue of $P_0 = Q_{ap} P Q_{ap}^*$, which depends on the ``outer'' rotation $Q_{ap}$.
		
		We substitute the expectation over $Q_{ap}$ with an expectation over the projected eigenvalues of the rotated precision matrix $P_0$:
		\eqal{&
			\EE_{Q_{ap} \sim p(Q)}[\EE_{Q_a \sim p(Q_a|Q_{ap})}[((P_{0,aa})^{-1})_{ii}] = \tr((Q_{ap} P Q_{ap}^*)^{-1})] \\&
			= \EE_{a_1, \dots, a_{D/2} \sim \nu_{A,D/2}(a_1, \dots, a_{D/2}|\lambda_1^{-1}, \dots, \lambda_D^{-1})}[a_1^{-1} + \dots + a_{D/2}^{-1}].
		}
		Here $X = (\lambda_1^{-1}, \dots \lambda_D^{-1})$ contains the eigenvalues of the precision matrix $P$, the inverse of the covariance $\Sigma$.
		\cref{lem:expected-projected-inverse-trace} with $K = D/2$ tells us how to evaluate the above expression.
		Insert the result into \cref{lem:active-scaling-by-precision-eigenvalues} to obtain the result.
	\end{proof}

	\subsubsection{Summary}
	
	We can now collect the above pieces to build the \textbf{proof of \cref{thm:single-layer-precise}}:
	\begin{proof}
		\cref{eq:expected-non-standardness-after} is the version of the non-Standardness after a single layer when \cref{as:centered,as:normalized-covariance} are fulfilled.
		Insert  \cref{lem:passive-scaling-contribution} (passive part) and \cref{lem:active-scaling-contribution} (active part) to obtain the result. The former required \cref{as:normalized-covariance} and the latter \cref{as:unitary-rotation,as:non-degenerate-eigenvalues} to hold.
	\end{proof}

	\subsubsection{Handling of imaginary part}
	\label{app:imaginary-part}
	
	If we allow for unitary rotations $Q \in U(D)$, real-valued data is typically rotated into imaginary space. In fact, the case that the input remains real even has probability zero:
	\eql{
		\PP[Qx \in \RR^D] = 0.
	}
	This does not pose a problem for our theory: The covariance matrix is positive definite also for complex data and so it has a positive determinant and trace, which are the only quantities entering the non-Standardness~$\Ss$ (see \cref{eq:non-standardness}).

	\subsection{Proof of \cref{thm:single-layer-guarantee}}
	\label{app:single-layer-guarantee-proof}
	
	\begin{lemma}
		\label{lem:active-scaling-contribution-var-max}
		With the definitions in \cref{sec:single-block-guarantees} and $p(Q)$ the Haar measure over the orthogonal group $O(D)$:
		\eql{
			\EE_{Q \sim p(Q)}[\log\det M_a^2] \geq D/2 \log\left(1 - \frac{DD}{2(D+1)(D-1)} \frac{\Var[\lambda]}{\lambda_{\max}}\right).
		}
	\end{lemma}
	\begin{proof}
		By \cref{prop:single-layer-whitening}, $M_a^2$ is given by:
		\eql{
			M_a^2 = \Diag(\Sigma_{0,aa} - \Sigma_{0,ap}\Sigma_{0,pp}^{-1}\Sigma_{0,pa})^{-1}.
		}
		The determinant of a diagonal matrix is equal to the product of the entries on the diagonal. By the permutation symmetry of $p(Q)$, we can pick the entry in the upper left corner:
		\eql{
			\EE_{Q \sim p(Q)}[\log\det M_a^2] = D/2 \EE_{Q \sim p(Q)}[\log ((M_a^{-2})_{11})] \leq -D/2 \log \EE_{Q \sim p(Q)}[(M_a^2)_{11}].
		}
		The last step is due to the Jensen inequality.
		
		We are left with computing $\EE_{Q \sim p(Q)}[(M_a^2)_{11}]$:
		\eqal{
			\EE_{Q \sim p(Q)}[(M_a^2)_{11}] &
			= \EE_{Q \sim p(Q)}[(\Sigma_{0,aa} - \Sigma_{0,ap}\Sigma_{0,pp}^{-1}\Sigma_{0,pa})_{11}] \\&
			= \EE_{Q \sim p(Q)}[(\Sigma_{0,aa})_{11}] - \EE_{Q \sim p(Q)}[(\Sigma_{0,ap}\Sigma_{0,pp}^{-1}\Sigma_{0,pa})_{11}] \\&
			= \frac1D \tr\Sigma_{0} - \EE_{Q \sim p(Q)}[(\Sigma_{0,ap}\Sigma_{0,pp}^{-1}\Sigma_{0,pa})_{11}].
		}
		The first expectation can be exactly computed via \cref{lem:helper-identities}.
		
		The average trace of the second matrix is not so easy to evaluate.
		As $\Sigma_{0,pp}^{-1}$ is positive definite, we can replace it with the worst case in the expectation:
		\eql{
			(\Sigma_{0,ap}\Sigma_{0,pp}^{-1}\Sigma_{0,pa})_{11} \geq (\Sigma_{0,ap}\lambda_{\max}^{-1} I \Sigma_{0,pa})_{11} = (\Sigma_{0,ap} \Sigma_{0,pa})_{11} \lambda_{\max}^{-1}.
		}
		$\lambda_{\max}$ is the largest eigenvalue of $\Sigma$, which does not depend on $Q$.
		
		The expectation value can now be computed exactly:
		\eqal{
			\EE_{Q \sim p(Q)}[(\Sigma_{0,ap} \Sigma_{0,pa})_{11}] &
			= \sum_{i=1}^{D/2} \EE_{Q \sim p(Q)}[(\Sigma_{0,ap})_{1i} (\Sigma_{0,pa})_{i1}] \\&
			= D/2 \EE_{Q \sim p(Q)}[(\Sigma_{0,ap})_{11}^2].
		}
		The last step is because each summand will have the same contribution.
		Writing the matrix multiplication out explicitly:
		\eql{
			(\Sigma_{0,ap})_{11}^2
			= (Q \Diag(\lambda_1, \dots, \lambda_D) Q^*)_{11}^2
			= (\sum_{j=1}^D Q_{1j} \lambda_j Q_{(D/2+1)j})^2
		}
		Again by symmetry, we can exchange axes and write $2$ instead of $D/2 + 1$ in what follows:
		\eql{
			(\sum_{j=1}^D Q_{1j} \lambda_j Q_{(D/2+1)j})^2
			= (\sum_{j=1}^D Q_{1j} \lambda_j Q_{2j})^2
			= \sum_{j,k=1}^D \lambda_j \lambda_k Q_{1j} Q_{2j} Q_{1k} Q_{2k}.
		}
		Taking the expectation, we use the linearity of the expectation and are left with the following monomials of elements of $Q$:
		\begin{enumerate}
			\item $j=k$: $\EE_{Q \sim p(Q)}[Q_{1j}^2 Q_{2j}^2] = \EE_{Q \sim p(Q)}[Q_{11}^2 Q_{21}^2]$ as we can exchanges axes,
			\item $j \neq k$: $\EE_{Q \sim p(Q)}[Q_{1j} Q_{2j} Q_{1k} Q_{2k}] = \EE_{Q \sim p(Q)}[Q_{11} Q_{21} Q_{12} Q_{22}]$ as we can exchange axes.
		\end{enumerate}
		By \cite{gorin_integrals_2002}, these amount to the following integrals of monomials of entries of random orthogonal matrices and the corresponding values:
		\eqal{
			1.&~
			\left\langle
			\begin{matrix}
				2 & 2
			\end{matrix}
			\right\rangle = \frac{1}{D(D+2)},
			\\
			2.&~
			\left\langle
			\begin{matrix}
				1 & 1 \\
				1 & 1
			\end{matrix}
			\right\rangle = -\frac{1}{D(D-1)(D+2)}.
		}
		
		Together, we find
		\eqal{
			\EE_{Q \sim p(Q)}[(M_a^2)_{11}] &
			= 1 - \frac{1}{2(D+2) \lambda_{\max}} \left(\sum_{j=1}^D \lambda_j^2 - \frac1{D-1} \sum_{j \neq k} \lambda_j \lambda_k \right) \\&
			= 1 - \frac{D^2}{2(D-1)(D+2)} \frac{\Var[\lambda]}{\lambda_{\max}}.
		}
		Here, $\Var[\lambda] = \frac1D \tr \Sigma^2 - (\frac1D \tr \Sigma)^2$ is the variance of the eigenvalues of $\Sigma$.
		
		Insert this to obtain the result.
	\end{proof}
	
	\begin{lemma}
		\label{lem:active-scaling-contribution-loss-only}
		With the definitions in \cref{sec:single-block-guarantees}:
		\eql{
			\EE_{Q \sim p(Q)}[\log\det M_a^2] \geq D/2 \log\left(1 - \frac{DD}{2(D+1)(D-1)} \frac{\Var[\lambda]}{\lambda_{\max}}\right).
		}
	\end{lemma}
	\begin{proof}
		The idea is to lower bound
		\eql{
			\frac{\Var[\lambda]}{\lambda_{\max}}
		}
		by some function of $L$. We make use of following arithmetic mean-geometric mean (AM-GM) inequality by \cite{cartwright_refinement_1978}:
		\eql{
			\label{eq:am-gm-inequality}
			\frac{\Var[\lambda]}{2 \lambda_{\max}} \leq \bar\lambda - g \leq \frac{\Var[\lambda]}{2 \lambda_{\min}},
		}
		where $g$ is the geometric mean of the eigenvalues:
		\eql{
			g := \left( \prod_{i=1}^D \lambda_i \right)^{1/D}.
		}
		We can write the loss $L$ directly via $g$ and vice versa:
		\eqal{
			L &= -\frac12 \log g^D = -\frac{D}{2} \log g, \\
			g &= \exp(-2L/D) \label{eq:geometric-mean-by-loss}.
		}
		
		Rewrite \cref{eq:am-gm-inequality} to our needs:
		\eql{
			\frac{\Var[\lambda]}{\lambda_{\max}}
			= \frac{\Var[\lambda] \lambda_{\min}}{\lambda_{\max} \lambda_{\min}}
			= \frac2\kappa \frac{\Var[\lambda]}{2\lambda_{\min}}
			\geq \frac2\kappa (1 - g),
		}
		with $\kappa$ the condition number of the covariance $\Sigma$.
		
		As we want a bound that merely depends on the loss, we upper bound $\kappa$ using a function of the loss, yielding a lower bound on $\Var[\lambda] / \lambda_{\max}$ that merely depends on the loss.
		The maximum of the condition value is given by:
		\eql{
			\max_{\substack{\lambda_1, \dots, \lambda_D \\ \sum_i \lambda_i = D \\ \prod_i \lambda_i^{1/D} = g}} \kappa = \frac{1 + \sqrt{1 - g^D}}{1 - \sqrt{1 - g^D}}.
		}
		
		This yields the required lower bound:
		\eql{
			\frac{\Var[\lambda]}{\lambda_{\max}}
			\geq 2\frac{1 - \sqrt{1 - g^D}}{1 + \sqrt{1 - g^D}} (1 - g),
		}
		which results in an overall upper bound:
		\eql{
			\EE_{Q \in O(D)}[\Ss(\Sigma_1(Q))] \leq \Ss(\Sigma) + \frac{D}{4} \log\left( 1 - \frac{D^2}{(D-1)(D+2)} \frac{1 - \sqrt{1 - g^D}}{1 + \sqrt{1 - g^D}} (1 - g) \right).
		}
		Replacing the expression in \cref{eq:geometric-mean-by-loss} for $g$ yields the statement.
	\end{proof}
	
	We summarize to obtain the \textbf{proof of \cref{thm:single-layer-guarantee}}:
	\begin{proof}
		\cref{eq:expected-non-standardness-after} is the form of non-Standardness $\Ss(\Sigma_1)$ (see \cref{eq:non-standardness}) we need to evaluate when \cref{as:centered,as:normalized-covariance} hold.
		Into this equation, insert \cref{lem:passive-scaling-contribution} together with \cref{lem:active-scaling-contribution-var-max} for the first bound. For the second bound, insert \cref{lem:passive-scaling-contribution,lem:active-scaling-contribution-loss-only}.
	\end{proof}

	\subsection{Proof of \cref{thm:deep-network-guarantee}}
	\label{app:deep-network-guarantee-proof}
	
	\begin{proof}
		The non-Standardness will not increase by the action of a single layer given in \cref{prop:single-layer-whitening} (compare \cref{eq:non-standardness-after}). This holds regardless of the rotations of the individual blocks $Q_1, \dots Q_L$, so $\Ss(\Sigma) = \Ss(\Sigma_0) \geq \Ss(\Sigma_1) \geq \dots \geq \Ss(\Sigma_L)$.
		It is easy to see that $\gamma$ decreases as $\Ss$ decreases by using $\Ss > 0$ to check that
		\eql{
			\frac{\partial \gamma}{\partial \Ss} > 0.
		}
		Together, we have:
		\eql{
			\gamma\left(\Ss(\Sigma_{L-1})\right) \leq \dots \leq \gamma\left(\Ss(\Sigma_{0})\right).
		}
		Rewrite \cref{thm:single-layer-guarantee} as follows:
		\eql{
			\EE_{Q \in O(D)}[\Ss(\Sigma_1(Q))] \leq \gamma\left(\Ss(\Sigma_{0})\right) \Ss(\Sigma),
		}
		and apply repeatedly:
		\eqal{
			\EE_{Q_1, \dots, Q_{L} \in O(D)}[\Ss(\Sigma_{L})] &
			\leq \EE_{Q_1, \dots, Q_{L-1} \in O(D)}[\gamma(\Ss(\Sigma_{L-1})) \Ss(\Sigma_{L-1})] \\&
			\leq \gamma(\Ss(\Sigma)) \EE_{Q_1, \dots, Q_{L-1} \in O(D)}[\Ss(\Sigma_{L-1})] \\&
			\leq \dots
			\leq \gamma(\Ss(\Sigma))^L \Ss(\Sigma_{0})
		}
		This shows the statement.
	\end{proof}

	\section{Compatible coupling architectures}
	\label{app:architectures}
	
	All statements in this paper apply to the following architectures, where we assume each layer to be equipped with ActNorm \cite{kingma_glow_2018}. To shorten the notation, we consider how a single dimension is transformed and rewrite the dependence on $p_0$ via a parameter vector $\theta = \theta(p_0)$, which is usually computed by a feed-forward neural network:
	\eql{
		y = c(x; \theta),
	}
	short for:
	\eql{
		(a_1)_i = c\left( (a_0)_i; \theta_i(p_0) \right).
	}
	
	\begin{itemize}
		\item \textbf{Affine coupling flows} in the form of NICE \cite{dinh_nice_2015}, RealNVP \cite{dinh_density_2017} and GLOW \cite{kingma_glow_2018}:
		\eql{
			c(x; \theta) = s x + t.
		}
		Here, $\theta = [s; t] \in \RR_+ \times \RR$. %
		
		\item \textbf{Nonlinear squared flow} \cite{ziegler_latent_2019}:
		\eql{
			c(x; \theta) = ax + b + \frac{c}{1 + (dx + h)^2},
		}
		for $\theta = [a, b, c, d, h] \in \RR^5$.
		
		\item \textbf{SOS polynomial flows} \cite{jaini_sum--squares_2019}:
		\eql{
			c(x; \theta) = \int_0^x \sum_{\kappa=1}^k \left(\sum_{l=0}^r a_{l,\kappa} u^l \right)^2 \d u + t.
		}
		Here, $\theta = [t; (a_{l, \kappa})_{l, \kappa}] \in \RR \times \RR^{r k}$.
		
		\item \textbf{Flow++} \cite{ho_flow_2019}:
		\eql{
			c(x; \theta) = s \sigma^{-1}\left( \sum_{j=1}^K \pi_j \sigma\left(\frac{x - \mu_j}{\sigma_j}\right) \right) + t.
		}
		Here, $\theta = [s; t; (\pi_j, \mu_j, \sigma_j)_{j=1}^K] \in \RR_+ \times \RR \times (\RR \times \RR \times \RR_+)^K$ and $\sigma$ is the logistic function.
		
		\item \textbf{Spline flows} in the form of piecewise-linear, monotone quadratic \cite{muller_neural_2019}, cubic \cite{durkan_cubic-spline_2019}, and rational quadratic \cite{durkan_neural_2019} splines. Here, $c$ is a spline of the corresponding type, parameterized by knots $\theta$.
		
		\item \textbf{Neural autoregressive flow} \cite{huang_neural_2018} parameterize $c(x; \theta)$ by a feed-forward neural network, which can be shown to be bijective if all weights are positive and all activation functions are strictly monotone.
		
		One can also restrict the neural network $c(x; \theta)$ to have positive output and integrate it numerically. This was introduced as \textbf{unconstrained monotonic neural networks} \cite{wehenkel_unconstrained_2019}.
	\end{itemize}

\end{document}